\documentclass[12pt]{article}
\usepackage{amsmath,amssymb,amsthm, amstext} 
\usepackage{bbm}
\usepackage{stmaryrd}
\usepackage{commath}
\usepackage[utf8]{inputenc}
\usepackage{enumitem}
\usepackage{dsfont}
\usepackage{natbib}
\usepackage{mathrsfs} 
\usepackage{mathtools}
\mathtoolsset{showonlyrefs}
\usepackage[ruled]{algorithm2e}
\usepackage{appendix}
\usepackage{float}

\newenvironment{policy}[1][htb]
{%
\begin{algorithm}[#1]%
}{\end{algorithm}}

\usepackage{mathrsfs} 
\usepackage[colorlinks, pdfstartview=FitH, citecolor = black, linkcolor  = black]{hyperref}

\usepackage{graphicx}
\usepackage{setspace}
\setdisplayskipstretch{}

\renewcommand{\P}{\mathbb{P}}

\renewcommand{\rho}{\varrho}

\newcommand\eps{\varepsilon}

\newcommand\N{\mathbb{N}}
\newcommand\E{\mathbb{E}}
\newcommand\R{\mathbb{R}}

\DeclareMathOperator*{\argmax}{arg\,max}

\newtheorem{theorem}{Theorem}[section]  
\newtheorem{assumption}[theorem]{Assumption}

\newtheorem{corollary}[theorem]{Corollary}

\newtheorem{lemma}[theorem]{Lemma}

\theoremstyle{definition}

\newtheorem{remark}[theorem]{Remark}

\begin{document}
\title{Functional Sequential Treatment Allocation with Covariates
}

\author{
\begin{tabular}{c}
Anders Bredahl Kock \\ 
\small	University of Oxford \\
\small	CREATES, Aarhus University\\
\small	{\small	\href{mailto:anders.kock@economics.ox.ac.uk}{anders.kock@economics.ox.ac.uk}} 
\end{tabular}
\and
\begin{tabular}{c}
David Preinerstorfer \\ 
{\small	ECARES, SBS-EM} \\ 
{\small	 Universit\'e libre de Bruxelles} \\
{\small	 \href{mailto:david.preinerstorfer@ulb.ac.be}{david.preinerstorfer@ulb.ac.be}}
\end{tabular}
\and
\begin{tabular}{c}
Bezirgen Veliyev \\ 
{\small	CREATES} \\ 
\small Aarhus University \\ 
{\small	\href{mailto:bveliyev@econ.au.dk}{bveliyev@econ.au.dk}}
\end{tabular}
}
\date{}

\date{First version: December 2018 \\
This version: January 2020}
\maketitle
\onehalfspacing
\begin{abstract}
We consider a multi-armed bandit problem with covariates. Given a realization of the covariate vector, instead of targeting the treatment with highest conditional expectation, the decision maker targets the treatment which maximizes a general \emph{functional} of the conditional potential outcome distribution, e.g., a conditional quantile, trimmed mean, or a socio-economic functional such as an inequality, welfare or poverty measure. We develop expected regret lower bounds for this problem, and construct a near minimax optimal assignment policy. 
\end{abstract}


\medskip \noindent \textbf{Keywords}: Sequential Treatment Allocation, Multi-Armed Bandit, Distributional Characteristics, Covariates, Minimax Optimal Expected Regret.

%

\newpage
\doublespacing

\section{Introduction}\label{sec:Intro}

The classical multi-armed bandit literature considers a sequential decision problem in which a policy maker attempts to assign subjects to the treatment with the highest expected outcome. Two practically relevant generalizations of this setting have attracted much attention: (i) a problem where the decision maker can incorporate a vector of covariates in the assignment of each subject, cf.~\cite{woodroofe1979one}, \cite{yang2002randomized}, \cite{rigollet2010nonparametric} and~\cite{perchet2013multi}; (ii) problems where instead of targeting the outcome distribution with highest expectation, the decision maker is interested in targeting another functional such as a quantile, a risk measure, or other characteristics of the distribution,~cf.~\cite{maillard}, \cite{NIPS2012_4753}, \cite{7515237}, \cite{vakili2018decision}, \cite{zimin2014generalized}, \cite{kock2017optimal}, \cite{tran2014functional}, \cite{cassel2018general}. Particularly relevant for the present article is the recent paper \cite{kpv1}, where a general theory is built for functional assignment problems albeit without covariates. 

While both types of generalizations have been well studied in isolation, the only article we are aware of to consider a multi-armed bandit problem with a target other than the conditional expectation in the presence of covariates is \cite{kock2017optimal}. That paper has two limitations: First, it considers the special class of functionals which can be written as a function of the conditional mean and the conditional variance. Therefore, many fundamental functionals are not covered by their theory, e.g., conditional quantiles or trimmed means. Secondly, regret lower bounds for functional targets (beyond the mean) are not discussed, and thus the question whether the algorithm they suggest is optimal remains open.

The goal of the present article is to develop a minimax expected regret optimality theory for multi-armed-bandit problems with functional targets and covariates. The regret function we work with is cumulative, i.e., every subject not assigned to the best treatment leads to a loss that cannot be offset by later assignments. The worst-case growth rate of the expected regret of a policy is thus linear in the number of assignments. 

The structure of the paper is as follows: The framework is discussed in Section~\ref{sec:setupX}. Here we first show that to obtain sublinear maximal expected regret it is not enough to assume that the conditional potential outcome distributions depend equicontinuously on the covariates. This insight motivates us to work with a minimally stronger Hölder-equicontinuity condition. As a consequence, even a slight relaxation of this assumption implies that every policy incurs the worst-case linear maximal expected regret. We also show that if a policy does not incorporate covariate information, then its regret grows linearly. In Section~\ref{sec:FUCBX} we introduce the functional upper-confidence-bound (F-UCB) policy in the presence of covariates. This is a binned version of the F-UCB policy introduced in~\cite{kpv1}, the binning being inspired by the UCBogram of \cite{rigollet2010nonparametric}. We then establish regret upper bounds for the F-UCB policy and obtain lower bounds, proving its near minimax expected regret optimality. The lower bounds are established under an assumption that essentially only requires the functionals not to be constant over the set of potential outcome distributions considered. This requirement is very weak, and thus guarantees that the lower bounds hold even under quite stringent restrictions on the conditional outcome distributions.

We first obtain such bounds without restricting the similarity of the best and second best treatment. Intuitively, however, this similarity crucially influences the difficulty of the decision problem and is therefore an important component. In particular, one may ask whether the F-UCB policy automatically ``adapts'' in an optimal way to the degree of similarity. To regulate the degree of similarity, we work with a version of the ``margin-condition'' tailored towards our functional target; similar conditions have been used in~\cite{mammen1999smooth}, \cite{tsybakov2004optimal}, \cite{audibert2007fast}, \cite{perchet2013multi}, and \cite{rigollet2010nonparametric}, the latter article, albeit targeting the conditional expectation, being particularly important for our developments. We first derive an upper bound on maximal expected regret of the F-UCB policy over subclasses of distributions that---besides the above-mentioned Hölder condition---satisfy the margin condition. We then establish nearly matching lower bounds over the just-mentioned classes of distributions. Finally, we show that the expected number of suboptimal assignments made increases as slowly as possible in the number of assignments. The latter result can be interpreted as an ethical guarantee on the F-UCB policy: only few persons will receive a treatment which is not optimal for them. The proofs can be found in the appendices.
 
\section{The setup and two impossibility results}\label{sec:setupX}

The observational structure in this paper is the one of a multi-armed bandit problem with covariates. That is, the subjects to be treated~$t = 1, \hdots, n$ arrive sequentially, and have to be assigned to one out of~$K \geq 2$ treatments. The assignment decision can incorporate previously observed outcomes, covariates and randomization. We denote the potential outcome of assigning subject~$t$ to treatment~$i$ by~$Y_{i,t}$, and assume throughout that~$a \leq Y_{i,t} \leq b$, where~$a<b$ are real numbers. The vector of potential outcomes is denoted as~$Y_t = (Y_{1, t}, \hdots, Y_{K, t})$; note that per subject only one coordinate of this vector can be observed. The covariate vector that comes with subject~$t$ is denoted by~$X_t$, and we assume throughout that~$X_t\in [0,1]^d$. Furthermore, for every~$t$, we let~$G_t$ be a random variable, which can be used for randomization in assigning the $t$-th subject. \emph{Throughout this article, we assume that~$(Y_t, X_t) = (Y_{1,t}, \ldots, Y_{K, t}, X_t)$ for~$t \in \N$ are i.i.d.; and we assume that the sequence of randomizations~$G_t$ is i.i.d., and is independent of the sequence~$(Y_t, X_t)$.} The distribution of~$G_t$ will be referred to as the randomization measure, which we think about as being fixed, e.g., the uniform distribution on~$[0, 1]$. Note that the dependence structure within each~$Y_t$ is not restricted.
We denote the distribution of~$(Y_t, X_t)$ as~$\P_{Y, X}$, and let~$\mathbb{P}_X$ be the marginal distribution of~$X_t$. The conditional cumulative distribution function (cdf) of~$Y_{i,t}$ given~$X_t = x$ is defined as~$F^i(y, x) = \mathsf{K}^i((-\infty, y], x)$, where~$\mathsf{K}^i: \mathcal{B}(\R) \times [0, 1]^d \to [0,1]$ denotes a regular conditional distribution of~$Y_{i,t}$ given~$X_t$, where~$\mathcal{B}(\R)$ are the Borel sets of~$\R$. We shall often impose the following condition (cf.~Remark~\ref{sec:discrete} for a discussion of discrete covariates). 
\begin{assumption}\label{ass:DensityX}
The distribution~$\P_X$ is absolutely continuous w.r.t.~Lebesgue measure on~$[0, 1]^d$, with a density that is bounded from below and above by~$\underline{c}>0$ and~$\overline{c}$, respectively.
\end{assumption}


A policy~$\pi$ is a triangular\footnote{We allow a policy to incorporate~$n$, because a decision maker who knows the number of subjects to be assigned might want to incorporate this into the assignment mechanism.} array~$\{\pi_{n,t}: n \in \N, 1 \leq t\leq n \}$, where the assignment of the $t$-th subject~$\pi_{n,t}$ takes as input the covariates~$X_t$, previously observed outcomes, covariates and randomizations (i.e., the complete observational history), and a randomization~$G_t$. We therefore have
\begin{equation}\label{eqn:policydef}
\pi_{n,t}: [0,1]^d \times \left[  [a,b] \times [0,1]^{d} \times \R \right]^{t-1} \times \R \to \mathcal{I}.
\end{equation}
Given a policy~$\pi$ and~$n\in \N$, the input to~$\pi_{n,t}$ is denoted as~$(X_t, Z_{t-1}, G_t)$, where~$Z_{t-1}$ is defined recursively: The first treatment~$\pi_{n,1}$ is a function of~$(X_1, Z_{0}, G_1) = (X_1, G_1)$. The second treatment is a function of~$X_2$, of~$Z_1:=(Y_{\pi_{n,1}(X_1, Z_0, G_1),1}, X_1, G_1)$, and of~$G_2$. For~$t \geq 3$ we have~$$Z_{t-1}:=(Y_{\pi_{n,t-1}(X_{t-1}, Z_{t-2},G_{t-1}), t-1}, X_{t-1}, G_{t-1}, Z_{t-2}).$$ 
The~$(t-1)(d+2)$-dimensional random vector~$Z_{t-1}$ can be interpreted as the information available after the~$(t-1)$-th treatment outcome has been observed.

In the present article treatments are evaluated according to a functional~$\mathsf{T}$ (e.g., the median) of the conditional potential outcome distribution, where the conditioning is on the covariates: The best assignment for a subject with covariate vector~$x \in [0, 1]^d$ is defined as
\begin{align*}
\pi^{\star}(x) = \min \argmax_{i\in\mathcal{I}}\mathsf{T}(F^i(\cdot, x)),
\end{align*}
where the minimum has been taken as a concrete choice of breaking ties. 

 We denote the ``parameter''-space of all potential conditional cdfs~$F^i(\cdot, x)$ by~$\mathscr{D}$. More precisely, we assume that
\begin{equation}\label{eqn:marginDincls}
\{F^i(\cdot, x): i = 1, \hdots, K \text{ and } x \in [0,1]^d \} \subseteq \mathscr{D},
\end{equation} 
where~$\mathscr{D}$ is a potentially large and nonparametric subset of~$D_{cdf}([a,b])$, the latter denoting the set of all cdfs on~$\R$ satisfying~$F(a-) = 0$ and~$F(b) = 1$. Note that the set~$\mathscr{D}$ encodes the assumptions one is willing to impose on the conditional outcome distributions.

The main assumption on~$\mathsf{T}$ we work with in the present paper is a Lipschitz-type condition first introduced in~\cite{kpv1} in the non-covariate setting. The assumption takes the following form, where~$\|\cdot\|_{\infty}$ denotes the supremum metric on the set of cdfs on~$\R$. For further discussion of the assumption see Remarks~2.3-2.5 in \cite{kpv1}.
\begin{assumption}\label{as:MAIN}
The functional~$\mathsf{T}: D_{cdf}([a,b]) \to \R$ and the non-empty set~$\mathscr{D} \subseteq D_{cdf}([a,b])$ satisfy
\begin{equation}\label{eqn:lipcondAS}
|\mathsf{T}(F) - \mathsf{T}(G)|\leq C\|F- G\|_{\infty} \quad \text{ for every } \quad F \in \mathscr{D} \text{ and every } G \in D_{cdf}([a,b])
\end{equation}
for some~$C > 0$.
\end{assumption}
As discussed at length in Appendices C and E of~\cite{kpv1}, under suitable assumptions on~$\mathscr{D}$, Assumption~\ref{as:MAIN} is satisfied, e.g., for quantiles, (trimmed) U-functionals, generalized L-functionals (cf.~\cite{serflinggen}), and many inequality-, poverty-, and welfare-measures important for socio-economic decision making. We keep the functional abstract in the present paper and refer the interested reader to the just-mentioned appendices for examples and detailed discussions. Apart from Assumption~\ref{as:MAIN}, we shall also impose the following measurability condition which does not impose any practical restrictions.
\begin{assumption}\label{as:MB}
For every~$m \in \N$, the function on~$[a,b]^m$ that is defined via~$x \mapsto \mathsf{T}(m^{-1} \sum_{j = 1}^m \mathds{1}\cbr[0]{x_j \leq \cdot}),$ i.e.,~$\mathsf{T}$ evaluated at the empirical cdf corresponding to~$x_1, \hdots, x_m$, is Borel measurable.
\end{assumption}

We shall now introduce the regret function used in the present paper to compare different policies. Given a policy~$\pi$, we define its (cumulative) regret as
\begin{align*}
R_n(\pi)
&=
R_n(\pi; F^1, \hdots, F^K, X_n, Z_{n-1}, G_n) \\
&=
\sum_{t=1}^n \left[ \mathsf{T} \big(F^{\pi^\star(X_t)} (\cdot, X_t) \big)- \mathsf{T}\big(F^{\pi_{n,t}(X_t, Z_{t-1}, G_t)} (\cdot, X_t) \big) \right].
\end{align*}

This regret function is ``individualistic'' in the sense that mistakes made for an individual cannot be compensated by later assignments. This property is attractive in settings where every individual matters. We note that for functional targets, and in the absence of covariates, other types of regret than the cumulative one have been considered. In particular~\cite{cassel2018general} consider a ``path-dependent'' regret notion, which is a useful alternative to cumulative regret. However, path-dependent regret seems to be very difficult to handle in the presence of covariates.

We evaluate policies based on their worst-case behavior, i.e., we shall study minimax expected regret properties of policies. Here, the maximum will be taken over sets of possible joint distributions~$\P_{Y,X}$. 

When establishing lower bounds on maximal expected regret we shall impose the following rather weak condition. It guarantees that there is a minimal amount of variation in the functional over a small subset of~$\mathscr{D}$, the set of all potential conditional outcome distributions.
\begin{assumption}\label{as:lbcov}
The functional~$\mathsf{T}: D_{cdf}([a,b]) \to \R$ satisfies Assumption~\ref{as:MAIN}, and~$\mathscr{D}$ contains two elements~$H_1$ and~$H_2$, such that 
\begin{equation*}
J_{\tau} := \tau H_1 + (1-\tau)H_2 \in \mathscr{D} \quad \text{ for every }\tau \in [0, 1],
\end{equation*}
and such that for some~$c_- > 0$ we have
\begin{equation}\label{eqn:sublinASX}
\mathsf{T}(J_{\tau_2}) - \mathsf{T}(J_{\tau_1}) \geq c_-(\tau_2-\tau_1) \quad \text{ for every }
\tau_1 \leq \tau_2 \text{ in } [0, 1].
\end{equation}
\end{assumption}
\noindent
We emphasize that Equation~\eqref{eqn:sublinASX} in Assumption~\ref{as:lbcov} is satisfied if, e.g., $\tau \mapsto \mathsf{T}(J_{\tau})$ is continuously differentiable on~$[0,1]$ with an everywhere positive derivative. 

Up to this point \emph{no} assumption has been imposed on the dependence of the conditional cdfs~$F^i(\cdot, x)$ on~$x \in [0, 1]^d$. Keeping this dependence unrestricted would allow two subjects with similar covariates to have completely different conditional outcome distributions. We now prove that the maximal expected regret of~\emph{any} policy increases linearly in~$n$ if the dependence of~$F^i(\cdot, x)$ on~$x$ is not further restricted. It even turns out that this statement continues to hold if one imposes the restriction that subjects with similar covariates have similar outcome distributions in the sense that
\begin{equation}\label{eqn:marginUEC}
\{F^i(y, \cdot): i = 1, \hdots, K  \text{ and }  y \in \R \} \quad \text{ is uniformly equicontinuous}.\footnote{The assumption in Equation~\eqref{eqn:marginUEC} imposes that: for every~$\varepsilon > 0$ there exists a~$\delta > 0$ such that~$\|x_1 - x_2\| \leq \delta$, for~$\|\cdot\|$ the Euclidean norm, implies~$|F^i(y, x_1) - F^i(y, x_2)| \leq \varepsilon$ for every~$i = 1, \hdots, K$ and every~$y \in \R$.}
\end{equation}
%
The theorem is as follows; it is obtained as an application of the lower bound developed in Theorem~\ref{LBCovMC} of Section~\ref{sec:margin}.
\begin{theorem}\label{thm:LN_UC}
Suppose~$K = 2$ and that Assumption~\ref{as:lbcov} is satisfied. Then there exists a constant~$c_l > 0$, such that for every policy~$\pi$ and any randomization measure, we have
\begin{equation*}
\sup \E [R_n({\pi})]\geq c_l n \quad \text{ for every } n \in \N,
\end{equation*}
where the supremum is taken over all~$(Y_t, X_t) \sim \P_{Y, X}$ for~$t = 1, \hdots, n$, where~$\P_{Y, X}$ satisfies Equations~\eqref{eqn:marginDincls} and~\eqref{eqn:marginUEC}, and where~$\P_X$ is the uniform distribution on~$[0, 1]^d$.
\end{theorem}
Since Assumption~\ref{as:MAIN} (which is a part of Assumption~\ref{as:lbcov}) implies that~$\mathsf{T}$ is bounded, Theorem~\ref{thm:LN_UC} shows that without imposing further restrictions beyond Equations~\eqref{eqn:marginDincls} and~\eqref{eqn:marginUEC} every policy incurs the worst case linear maximal expected regret. 
 
We shall from now on impose a Hölder equicontinuity condition on~$F^i(\cdot,x)$. This condition is only slightly stronger than uniform equicontinuity, but will turn out to be enough to ensure existence of (near) minimax optimal policies with nontrivial maximal expected regret. 

\begin{assumption} \label{ass:Holder}
There exist a~$\gamma \in (0,1]$ and an~$L>0$, such that for every~$i = 1, \hdots, K$ and every~$y \in \R$, we have
\begin{equation*}
|F^{i}(y, x_1)-F^{i}(y, x_2)| \leq L ||x_1-x_2||^{\gamma} \text{ for every } x_1, x_2 \in [0,1]^d.
\end{equation*}
\end{assumption}

Before studying policies that incorporate covariate information, one may wonder (e.g., as a sanity check of the framework considered) what happens if one uses a policy that ignores covariates. Our next result shows that---unless the underlying distribution~$\P_{Y, X}$ happens to be such that the covariates are completely irrelevant for the assignment problem---any policy that \emph{ignores covariates} must incur a linear expected regret. Formally, a policy~$\pi$ is said to ignore covariates, if there exists another double array~$\tilde{\pi}_{n,t}: \left[ [a,b]  \times \R \right]^{t-1} \times \R \to \mathcal{I}$ of measurable functions, such that, for every~$n$ and every~$t = 1, \hdots, n$, we have~$\pi_{n,t} = \tilde{\pi}_{n, t} \circ \Pi_t$, where the function~$\Pi_t$ projects every~$w = (x, z, g)$ in the domain of~$\pi_{n,t}$ to $(\tilde{z}, g)$,~$\tilde{z}$ being obtained from~$z \in \left[ [a,b] \times [0, 1]^d \times \R \right]^{t-1}$ by dropping the~$(t-1)$ coordinates taking values in~$[0, 1]^d$. Note that then,~$\pi_{n,t}(Z_{t-1}, G_t) = \tilde{\pi}_{n,t}(\tilde{Z}_{t-1}, G_t)$, where for~$t \geq 2$ we have
$\tilde{Z}_{t-1} = (Y_{\tilde{\pi}_{n,t-1}(\tilde{Z}_{t-2}, G_{t-1})}, G_{t-1}, \hdots , Y_{\tilde{\pi}_{n,1}(\tilde{Z}_{0}, G_{1})}, G_1))
$ and~$(\tilde{Z}_{0}, G_{1}) = G_1$.
\begin{theorem}\label{LBIgnoreCov}
Let~$K = 2$, suppose~$\mathsf{T}: D_{cdf}([a,b]) \to \R$ satisfies Assumption~\ref{as:MAIN}, and let~$\P_{Y, X}$ satisfy Equation~\eqref{eqn:marginDincls} and Assumption~\ref{ass:Holder}. Define the sets
\begin{align*}
&A_1 := \{x \in [0, 1]^d: \mathsf{T}(F^1(\cdot, x)) > \mathsf{T}(F^2(\cdot, x))\}, \\
&A_2 := \{x \in [0, 1]^d: \mathsf{T}(F^1(\cdot, x)) < \mathsf{T}(F^2(\cdot, x))\}.
\end{align*}
Then, there exists a~$c_l > 0$, such that for every policy~$\pi$ ignoring covariates, and any randomization measure, we have
\begin{equation}\label{MinMaxIneq}
\mathbb{E}[R_n(\pi)] \geq c_l \min(\P_X(A_1), \P_X(A_2)) n \quad \text{ for every } n \in \N.
\end{equation}
\end{theorem}
Thus, the expected regret of any policy ignoring covariates must increase at the worst-case linear  rate in~$n$, for \emph{any} distribution~$\P_{Y, X}$ for which the identity of the best treatment depends on the covariates in the sense that
\begin{equation*}
\min\left(\P_X(A_1), \P_X(A_2)\right) > 0.
\end{equation*}
Contrary to all other lower bounds established in this article, the lower bound in the previous theorem is valid even pointwise, as it makes a statement about any fixed distribution~$\P_{Y, X}$.

\section{The F-UCB policy in the presence of covariates}\label{sec:FUCBX}

We now introduce a version of the F-UCB policy that incorporates covariate information. This policy generalizes the UCBogram in \cite{rigollet2010nonparametric} from the conditional mean setting to the general functional setup. The underlying idea is to categorize subjects into groups, according to the similarity of their covariate vector, and to run, separately within each group, a policy targeting the treatment that is best for the ``average'' subject in each group. Two covariate vectors~$x_1$ and~$x_2$ are considered similar, if they fall into the same element of a given partition~$B_{n,1}, \hdots,B_{n, M(n)}$ of~$[0, 1]^d$, where every~$B_{n,i}$ is a non-empty Borel set. Targeting the~``on average''-best treatment for each group here means that for~$B_{n,j}$ with~$\P_{X}(B_{n,j}) > 0$ our policy targets a treatment that attains~$\max_{i\in\mathcal{I}}\mathsf{T}(F_{n,j}^i)$, where~$F_{n,j}^i$ is the conditional cdf of~$Y_{i,t}$ given~$X_t \in B_{n,j}$, i.e.,
\begin{align}\label{Fji}
F_{n,j}^i(y):=\frac{1}{\P_{X}(B_{n,j})}\int_{B_{n,j}} F^i(y,x)d\P_{X}(x).
\end{align}
Note that in general~$\argmax_{i \in \mathcal{I}} \mathsf{T}(F_{n,j}^i) \neq \argmax_{i \in \mathcal{I}} \mathsf{T}(F^i(\cdot, x))$. Targeting~$\max_{i \in \mathcal{I}} \mathsf{T}(F_{n,j}^i)$ hence results in a bias. The choice of the partition~$B_{n,1}, \hdots, B_{n,M(n)}$ needs to balance this bias against an increase in variance due to having fewer subjects in each group. This is akin to choosing a bandwidth to balance variance and bias terms in nonparametric estimation problems. 

In order to describe the F-UCB policy in the presence of covariates, we need to introduce the following notation. For any policy $\pi$ and~$B_{n,1}, \hdots, B_{n, M(n)}$ as above let
\begin{align*}
S^i_{n,j}(t)=\sum_{s=1}^t\mathds{1}_{\cbr[0]{X_s\in B_{n,j},\ \pi_{n, s}(X_s,Z_{s-1},G_s)=i}},
\end{align*}
be the number of times that it has assigned treatment $i$ to individuals with covariates in $B_{n,j}$ up to time $t$. On the event $\cbr[0]{S^i_{n,j}(t)>0}$ define  the empirical cdf based on the outcomes of all subjects in~$\{1, \hdots, t\}$ with covariates in $B_{n,j}$ that have been assigned to treatment~$i$ as
\begin{align*}
\hat{F}^i_{n,t,j}(z)=\frac{1}{S^i_{n,j}(t)}\sum_{s=1}^t\mathds{1}_{\cbr[0]{Y_{i,s}\leq z}}\mathds{1}_{\cbr[0]{X_s\in B_{n,j},\ \pi_{n,s}(X_s,Z_{s-1},G_s)=i}}.
\end{align*}
The F-UCB policy with covariates, $\bar{\pi}$, is described in Policy~\ref{poly:FUCBC}. We note that it amounts to using the F-UCB policy~$\hat{\pi}$, say, of~\cite{kpv1} locally on each $B_{n,j}$. Their policy was defined in a setting without covariates and it does not rely on external randomization. We refer the reader to~\cite{kpv1} for more details on the F-UCB policy in the absence of covariates. We shall in particular use their Theorem~4.1, which provides a regret upper bound for this policy in their setting. In~\cite{kpv1} it is also shown that choosing the tuning parameter $\beta=2+\sqrt{2}$, minimizes the constant in the uniform upper bounds on expected regret.

\bigskip
\onehalfspacing
\begin{policy}[H]
\caption{F-UCB policy with covariates~$\bar{\pi}$ \label{poly:FUCBC}}
\textbf{Inputs:}~$\beta > 2$, Partition~$B_{n,1}, \hdots,B_{n, M(n)}$ of~$[0, 1]^d$ into non-empty Borel sets\\
\textbf{Set:}~$N_j = 1$ for~$j = 1, \hdots, M(n)$\\
\For{$t = 1, \hdots, n$}{
\For{$j = 1, \hdots, M(n)$}{
\If{$X_t \in B_{n,j}$ and $N_j\leq K$
}{assign~$\bar{\pi}_t(X_t,Z_{t-1},G_t)=N_j$ \\
$N_j \leftarrow N_j +1$ }
\If{$X_t\in B_{n,j}$ and $N_j>K$}{assign
$\bar{\pi}_t(X_t,Z_{t-1},G_t)=\min\argmax\cbr[2]{\mathsf{T}(\hat{F}^i_{n, t-1,j})+C\sqrt{\beta \log(N_j)/(2S^i_{n,j}(t-1))}}$\\
$N_j\leftarrow N_j+1$}
}
}
\end{policy}
\doublespacing


\subsection{Upper bounds on the maximal expected regret of~$\bar{\pi}$ and a first lower bound}\label{sec:UpperBoundsNoMargin}

The following theorem gives an upper bound on the maximal expected regret of the F-UCB Policy~\ref{poly:FUCBC} in the presence of covariates, and for \emph{any} choice of partition. This flexibility may be useful since the policy maker is often constrained in the way groups can be formed. The result quantifies how the partitioning affects the regret guarantees. We denote~$\overline{\log}(x) := \max(1, \log(x))$ for~$x > 0$.
\begin{theorem}\label{Thm:Covariate}
Suppose Assumptions~\ref{as:MAIN} and~\ref{as:MB} hold. Assume further that~$\mathscr{D}$ is convex. Consider the F-UCB policy with covariates~$\bar{\pi}$, and let~$V_{n,j}=\sup_{x_1, x_2\in B_{n,j}}\enVert[0]{x_1-x_2}$ be the diameter of~$B_{n,j}$. Then, for~$c=c(\beta,C) = C \sqrt{ 2 \beta + (\beta+2)/(\beta-2)}$ it holds that
\begin{equation}\label{eq:regret_cov_gen}
\sup
\E[R_n(\bar\pi)]
\leq
\sum_{j=1}^{M(n)}
\left[c\sqrt{Kn\P_X(B_{n,j}) \overline{\log}(n\P_X( B_{n,j}))}+2CLV_{n,j}^\gamma n \P_X(B_{n,j})\right]  \text{ for every } n \in \N,
\end{equation}
where the supremum is taken over all~$(Y_t, X_t) \sim \P_{Y, X}$ for~$t = 1, \hdots, n$, where~$\P_{Y, X}$ satisfies Equation~\eqref{eqn:marginDincls}, and Assumption~\ref{ass:Holder} with~$L$ and~$\gamma$.\footnote{Here~$\P_X(B_{n,j}) \overline{\log}(n\P_X( B_{n,j}))$ is to be interpreted as~$0$ in case~$\P_X(B_{n,j}) = 0$.}
\end{theorem}

Each of the summands~$j = 1, \hdots, M(n)$ in the upper bound on the maximal expected regret in Equation~\eqref{eq:regret_cov_gen} consists of two parts: The first part is structurally very similar to the upper bound of Theorem~4.1 in~\cite{kpv1}, which the proof of the theorem draws on. The difference is that the total number of subjects to be treated,~$n$, has now been replaced by~$n\P_X(B_{n,j})$, the number of subjects expected to fall into~$B_{n,j}$. Inspection of the proof shows that the first part is the regret we expect to accumulate on~$B_{n,j}$, compared to always assigning the treatment that is best for the ``average subject'' in~$B_{n,j}$, i.e., compared to always assigning an element of~$\argmax_{i\in\mathcal{I}}\mathsf{T}(F_{n,j}^i)$, where we recall the definition of~$F_{n,j}^i$ from Equation~\eqref{Fji}. The second part in each summand in the upper bound in~\eqref{eq:regret_cov_gen} is a bias term: It is the approximation error incurred due to~$\bar{\pi}$ effectively targeting~$\max_{i\in\mathcal{I}}\mathsf{T}(F_{n,j}^i)$ instead of 
$\mathsf{T} \big(F^{\pi^\star(x)} (\cdot, x) \big)$ for every~$x \in B_{n,j}$. 

A frequently used class of partitions of~$[0,1]^d$ are hypercubes, which are obtained by hard thresholding each coordinate of~$X_t$. The so-created groups may not only result in low regret, but are also relevant due to their simplicity and resemblance to ways of grouping subjects in practice. More precisely, fix~$P\in \mathbb{N}$ and define for every~$k=(k_1,\hdots,k_d)\in\{1,\hdots,P\}^d$ the hypercube 
\begin{align} \label{SquareBins}
\left\{x\in [0, 1]^d:\frac{k_l-1}{P}\leq x_l  \preceq \frac{k_l}{P},\ l=1,\hdots,d\right\} ,
\end{align}
where~$\preceq$ is to be interpreted as~$\leq$ for~$k_l=P$, and as~$<$ otherwise. This defines a partition of~$[0,1]^d$ into~$P^d$ hypercubes with side length~$1/P$ each. We now order these hypercubes lexicographically according to their index vector~$k$, to obtain the corresponding~\emph{cubic partition}~$B_1^P, \hdots, B_{P^d}^P$. 
The following result specializes Theorem~\ref{Thm:Covariate} to this specific partition and for a choice of~$P$ that will be shown to be optimal below.
\begin{corollary}\label{Cor:SimpleBins}
Suppose Assumptions~\ref{as:MAIN} and~\ref{as:MB} hold. Assume further that~$\mathscr{D}$ is convex. Let~$\gamma \in (0, 1]$. Consider the F-UCB policy with covariates~$\bar{\pi}$, based on a cubic partition~$B_{n,j} = B^P_{j}$ for $j = 1, \hdots, M(n) = P^d$ as defined in Equation~\eqref{SquareBins}, and with~$P = \lceil n^{1/(2 \gamma +d)}  \rceil$. Then there exists a constant~$c = c(d, L, \gamma, \bar{c}, C, \beta)>0$, such that
\begin{align}
\sup \mathbb{E}\left[R_n(\bar\pi)\right]\leq c \sqrt{K \overline{\log}(n)}~ n^{1-\frac{\gamma}{2\gamma+d}}\label{part2} \quad \text{ for every } n \in \N,
\end{align}
where the supremum is taken over all~$(Y_t, X_t) \sim \P_{Y, X}$ for~$t = 1, \hdots, n$, where~$\P_{Y, X}$ satisfies Equation~\eqref{eqn:marginDincls}, Assumption~\ref{ass:DensityX} with~$\overline{c}$ (and any~$\underline{c}$), and Assumption~\ref{ass:Holder} with~$L$ and~$\gamma$.
\end{corollary}
Corollary~\ref{Cor:SimpleBins} reveals that it is possible to achieve sublinear (in~$n$) maximal expected regret under the Hölder equicontinuity condition imposed through Assumption~\ref{ass:Holder}. This is interesting also in light of Theorem~\ref{thm:LN_UC}, which showed that under the slightly weaker assumption of uniform equicontinuity, every policy has linearly increasing maximal expected regret. Hence, there is little room for weakening Assumption~\ref{ass:Holder}. Note that a ``curse of dimensionality'' is present, in the sense that the upper bound in Corollary~\ref{Cor:SimpleBins} gets close to linear in~$n$, as the number of covariates~$d$ increases. This is due to the fact that as a part of the regret minimization, one \emph{sequentially} estimates the conditional distributions~$F^i(y,\cdot)$ of the treatment outcomes, where each cdf is a function of~$d$ variables. Finally, we observe that the upper bound is increasing in the number of available treatments~$K$. Intuitively, this is because more observations must be used for experimentation when more treatments are available.

The partitioning used in Corollary~\ref{Cor:SimpleBins} results in a near-minimax optimal policy, as we show in the following theorem, which establishes a lower bound on maximal expected regret. The statement follows from Theorem~\ref{LBCovMC} in Section~\ref{sec:margin} below.
\begin{theorem}\label{thm:LB_NoMargin}
Suppose~$K = 2$ and that Assumption~\ref{as:lbcov} is satisfied. Let~$\gamma \in (0, 1]$. Then, for every~$\varepsilon \in (0, \gamma/(2\gamma + d))$, every policy~$\pi$ and any randomization measure, we have
\begin{equation*}
\sup \E [R_n({\pi})]\geq n^{1-\frac{\gamma}{2\gamma + d}} ~ n^{-\varepsilon} c_l(\varepsilon) \quad \text{ for every } n \in \N,
\end{equation*}
where the supremum is taken over all~$(Y_t, X_t) \sim \P_{Y, X}$ for~$t = 1, \hdots, n$, where~$\P_{Y, X}$ satisfies Equation~\eqref{eqn:marginDincls}, Assumption~\ref{ass:Holder} with parameters~$\gamma$ and~$L =  1/\sqrt{17}$,~$\P_X$ is the uniform distribution on~$[0, 1]^d$, and where~$$c_l^{-1}(\varepsilon) = 64^{1+1/\alpha(\varepsilon)}(8d(c_- 2L)^{-\alpha(\varepsilon)} +1)^{1/\alpha(\varepsilon)} \quad \text{ with } \quad \alpha(\varepsilon) = (2\gamma + d)\varepsilon/\gamma.$$
\end{theorem}
Comparing the lower bound on maximal regret in Theorem~\ref{thm:LB_NoMargin} to the upper bound on maximal expected regret established in Corollary~\ref{Cor:SimpleBins}, reveals that the F-UCB policy with a cubic partition and with~$P = \lceil n^{1/(2 \gamma +d)}  \rceil$ is near-optimal: If a policy with strictly smaller maximal expected regret exists, the order of improvement must be~$o(n^\eps)$ for all~$\varepsilon \in (0, \gamma/(2\gamma + d))$, e.g., logarithmic. In particular this also means that if nothing prohibits cubic partitioning, not much can be gained from a maximal expected regret point-of-view in searching for ``better'' partitions under the given set of assumptions. 

\begin{remark}[Unknown horizon and the doubling trick]
The policy~$\bar{\pi}$ with cubic partitioning~$P=\lceil n^{1/(2 \gamma +d)}  \rceil$, as considered in Corollary~\ref{Cor:SimpleBins}, can be used in practice only if one knows~$n$, i.e., the policy  is not anytime. If~$n$ is unknown, however, one can use the ``doubling trick'' to construct a policy with an upper bound on the maximal expected regret that is of the same order as in Corollary~\ref{Cor:SimpleBins}, but with higher multiplicative constants. In essence, the doubling trick works by ``restarting'' the policy at times~$2^m,\ m\in \N$. We refer to \cite{shalev2012online} and the recent work by~\cite{besson2018doubling} for more details.  
\end{remark}

\begin{remark}[Discrete covariates]\label{sec:discrete}
We mostly focus on the case of continuous covariates (although this is not formally required in Theorem~\ref{Thm:Covariate}). A natural, and also near minimax rate-optimal, solution to incorporate discrete covariates would be to fully condition on these, i.e., to apply the F-UCB policy of~\cite{kpv1} separately for each combination of discrete covariates. In the present article, we omit formal statements concerning discrete covariates, but we emphasize that corresponding results can be obtained by conditioning arguments.
\end{remark}

\subsection{Optimality properties under the margin condition}\label{sec:margin}

Besides mild conditions on~$\P_X$, our results so far have only assumed that the conditional distributions of the treatment outcomes are Hölder equicontinuous. In particular, the sets of distributions over which the F-UCB policy has been shown to be optimal does not restrict the (unknown) similarity of the best and second best treatment. In the present section, we shall see that in classes of distributions where the best and second best treatment are ``well-separated,'' the upper bound on maximal expected regret of the F-UCB policy can be lowered (without changing the policy), and that the F-UCB policy optimally adapts to the degree of similarity of the best and the remaining treatments. 

Besides being of interest in their own right, the results in the present section are instrumental to proving our impossibility result Theorem~\ref{thm:LN_UC} and to establishing the expected regret lower bound in Theorem~\ref{thm:LB_NoMargin}. 

To formally define the well-separateness condition we shall work with, we need to define for every~$x \in [0, 1]^d$ the second best treatment~$\pi^{\sharp}(x)$; note that  in principle there can be multiple treatments that are as good as the best treatment~$\pi^*(x)$. For~$x \in [0,1]^d$, if 
$\min_{i \in \mathcal{I}}\mathsf{T}(F^i(\cdot, x)) < \mathsf{T}\big(F^{\pi^{\star}(x)}(\cdot, x)\big)$, we define the second best treatment as
\begin{align*}
\pi^{\sharp}(x) := \min \argmax_{i \in \mathcal{I}} \left\{ \mathsf{T}(F^i(\cdot, x)): \mathsf{T}(F^i(\cdot, x))<\mathsf{T}\big(F^{\pi^{\star}(x)}(\cdot, x)\big) \right\};
\end{align*}
and we set~$\pi^{\sharp}(x)=1$ otherwise, i.e., if all treatments are equally good. We can now introduce the \emph{margin condition}.
\begin{assumption}\label{ass:Margin}
There exists an~$\alpha \in (0,1)$ and a~$C_0>0$, such that\footnote{We note that the events in the displayed equation of Assumption~\ref{ass:Margin} are not necessarily Borel measurable. Therefore, Assumption~\ref{ass:Margin} implicitly imposes measurability on all events considered. Note, however, that in case Assumptions~\ref{as:MAIN} and~\ref{ass:Holder} as well as the inclusion in Equation~\eqref{eqn:marginDincls} are assumed, this measurability condition is easily seen to be satisfied.}
\begin{equation*}
\P_X \left(x \in [0, 1]^d: 0<\mathsf{T}\big(F^{\pi^{\star}(x)}(\cdot,x) \big)-\mathsf{T} \big(F^{\pi^{\sharp}(x)}(\cdot,x) \big) \leq \delta   \right) 
\leq C_0 \delta^{\alpha} \mbox{ for all } \delta \in [0, 1].
\end{equation*}
\end{assumption}
The margin condition restricts how likely it is that the best and second best treatment are close to each other. In particular, it limits the probability of these two treatments being almost equally good, i.e., being within a~$\delta$-margin.  Assumptions of this type have previously been used in the works of \cite{mammen1999smooth}, \cite{tsybakov2004optimal}, and \cite{audibert2007fast} in the statistics literature. In the context of statistical treatment rules, the margin condition has recently been used in the work of \cite{kitagawa2018should}, who considered empirical welfare maximization in a static treatment allocation problem. Finally, the margin condition was used by \cite{rigollet2010nonparametric} and \cite{perchet2013multi} in the context of a multi-armed bandit problem targeting the conditional mean. The proofs of the results in the present section draw in particular their ideas.

Adding the margin condition, the maximal expected regret of the F-UCB policy based on cubic partitions can be bounded as follows. 
\begin{theorem}
\label{UBCovariateMargin}
Suppose Assumptions~\ref{as:MAIN} and~\ref{as:MB} hold. Assume further that~$\mathscr{D}$ is convex. Let~$\gamma \in (0, 1]$. Consider the F-UCB policy with covariates~$\bar{\pi}$, based on a cubic partition~$B_{n,j} = B^P_{j}$ for $j = 1, \hdots, M(n) = P^d$, as defined in Equation~\eqref{SquareBins}, and with~$P = \lceil n^{1/(2 \gamma +d)}  \rceil$. Then there exists a constant~$c = c(d, L, \gamma, \underline{c}, \bar{c}, C, C_0, \alpha, \beta) >0$, such that
\begin{align}\label{eqn:UBCovariateMargin}
\sup \mathbb{E}\left[R_n(\bar\pi)\right]\leq c K \overline{\log}(n) n^{1-\frac{\gamma(1+\alpha)}{2 \gamma+d}} \quad \text{ for every } n \in \N,
\end{align}
where the supremum is taken over all~$(Y_t, X_t) \sim \P_{Y, X}$ for~$t = 1, \hdots, n$, where~$\P_{Y, X}$ satisfies Equation~\eqref{eqn:marginDincls}, Assumption~\ref{ass:DensityX} with~$\underline{c}$ and~$\overline{c}$, Assumption~\ref{ass:Holder} with~$L$ and~$\gamma$, and Assumption~\ref{ass:Margin} with~$\alpha \in (0, 1)$ and~$C_0 > 0$.
\end{theorem}
Compared to Corollary~\ref{Cor:SimpleBins} the exponent on~$n$ in the upper bound on regret is smaller, the difference depending on~$\alpha$. Thus, in the presence of Assumption~\ref{ass:Margin}, the regret guarantee of the F-UCB policy is stronger, even without incorporating~$\alpha$ into the policy. We shall see in Theorem~\ref{LBCovMC} below that the upper bound on maximal expected regret in Theorem~\ref{UBCovariateMargin} is optimal in~$n$ up to logarithmic factors.

The margin condition also allows us to prove an upper bound on the expected number of suboptimal assignments made by the F-UCB policy. We shall define the total number of suboptimal assignments for a policy~$\pi$ over the course of a total of~$n$ assignments as
\begin{align*}
S_n(\pi) &= S_n(\pi; F^1, \hdots, F^K, X_n, Z_{n-1}, G_n) \\
&=
\sum_{t=1}^n\mathds{1}\cbr[1]{\pi_{n,t}(X_t,Z_{t-1}, G_t)\not\in\argmax\cbr[0]{\mathsf{T}(F^i(\cdot,X_t)): i=1,\hdots,K}}. 
\end{align*}
We now establish a uniform upper bound on~$\E[S_n(\bar{\pi})]$ for the F-UCB policy~$\bar{\pi}$ based on cubic partitions. 
\begin{theorem}\label{thm:ISR}
Suppose Assumptions~\ref{as:MAIN} and~\ref{as:MB} hold. Assume further that~$\mathscr{D}$ is convex. Let~$\gamma \in (0, 1]$. Consider the F-UCB policy with covariates~$\bar{\pi}$, based on a cubic partition~$B_{n,j} = B^P_{j}$ for $j = 1, \hdots, M(n) = P^d$, as defined in Equation~\eqref{SquareBins}, and with~$P = \lceil n^{1/(2 \gamma +d)}  \rceil$. Then there exists a constant~$c = c(d, L, \gamma, \underline{c}, \bar{c}, C, C_0, \alpha, \beta) >0$, such that
\begin{align}
\sup \mathbb{E}\left[S_n(\bar\pi)\right]\leq c [K \overline{\log}(n)]^{\frac{\alpha}{1+\alpha}}n^{1-\frac{\alpha\gamma}{2\gamma+d}} \quad \text{ for every } n \in \N,
\end{align}
where the supremum is taken over all~$(Y_t, X_t) \sim \P_{Y, X}$ for~$t = 1, \hdots, n$, where~$\P_{Y, X}$ satisfies Equation~\eqref{eqn:marginDincls}, Assumption~\ref{ass:DensityX} with~$\underline{c}$ and~$\overline{c}$, Assumption~\ref{ass:Holder} with~$L$ and~$\gamma$, and Assumption~\ref{ass:Margin} with~$\alpha \in (0, 1)$ and~$C_0 > 0$.
\end{theorem}
The upper bound in Theorem~\ref{thm:ISR} is a useful theoretical guarantee, because it limits the number of subjects who receive suboptimal treatments. As the last result in this section, we prove that the upper bounds in Theorems~\ref{UBCovariateMargin} and~\ref{thm:ISR} are near minimax optimal. This ensures, in particular, that the good behavior of the maximal expected regret of the F-UCB policy does not come at the price of excessive experimentation, leading to unnecessarily many suboptimal assignments. 

\begin{theorem}\label{LBCovMC}
Suppose~$K = 2$ and that Assumption~\ref{as:lbcov} is satisfied. Let~$\gamma \in (0, 1]$. Then for every policy~$\pi$ and any randomization measure, we have
\begin{equation}\label{eqn:suplow}
\sup \E [R_n({\pi})]\geq n^{1-\frac{\gamma(1+\alpha)}{2\gamma + d}} \big/ \left[64^{1+1/\alpha}(C_0 +1)^{1/\alpha}\right]   \quad \text{ for every } n \in \N, 
\end{equation}
and
\begin{equation}\label{eqn:suplow2}
\sup \E[ S_n({\pi}) ]\geq n^{1-\frac{\alpha\gamma}{d+2\gamma}} \big / 32 \quad \text{ for every } n \in \N, 
\end{equation}
where both suprema are taken over all~$(Y_t, X_t) \sim \P_{Y, X}$ for~$t = 1, \hdots, n$, where~$\P_{Y,X}$ satisfies 
Equation~\eqref{eqn:marginDincls}, Assumption~\ref{ass:Holder} with parameters~$\gamma$ and~$L =  17^{-1/2}$, Assumption~\ref{ass:Margin} with~$\alpha \in (0, 1)$ and~$C_0 = 8d(c_- 2L)^{-\alpha}$, and where~$\P_X$ is the uniform distribution on~$[0, 1]^d$.
\end{theorem}

Together with Theorem~\ref{UBCovariateMargin} the statement in Equation~\eqref{eqn:suplow} shows that the F-UCB policy is near minimax optimal in terms of maximal expected regret. Similarly, together with Theorem~\ref{thm:ISR} the lower bound in Equation~\eqref{eqn:suplow2} proves that the F-UCB policy assigns the minimal number of suboptimal treatments. The proof idea is classic and links regret to testing problems.  In particular, as in \cite{rigollet2010nonparametric} who target the conditional mean functional, we first use the margin condition to lower bound the expected regret by the expected number of false assignments (cf.~Lemma~\ref{ISR}). Then, we show that the expected number of false assignments can be lower bounded by sums of Type 1 and Type 2 errors of testing problems in certain binary experiments between elements of a subfamily of the joint distributions of~$(X_t, Y_t)$ over which the suprema in the previous theorem are taken. In order to get good lower bounds, we are required to work with a family of joint distributions (in particularly satisfying the assumption in Equation~\eqref{eqn:marginDincls} and Assumption~\ref{ass:Holder}) the elements of which are difficult to distinguish, while having sufficient variation in~$x \mapsto \mathsf{T}(F^i(\cdot, x))$. This constitutes one main complication compared to the argument in \cite{rigollet2010nonparametric}, where the functional is the conditional expectation, and where one can work with joint distributions where~$Y_t$ given~$X_t$ is a Bernoulli distribution with a certain success probability depending on the covariate vector. In our case Equation~\eqref{eqn:marginDincls} needs to be satisfied. The only assumption on~$\mathscr{D}$ we can work with is Assumption~\ref{as:lbcov}. Therefore, we need to choose the conditional distributions from the line segment provided in this Assumption. While \cite{rigollet2010nonparametric} construct joint distributions replicating conditional mean surfaces such that the above testing problems are difficult enough to warrant large lower bounds, we construct joint distributions replicating conditional functional surfaces that render the testing problems difficult; that is, we construct joint distribution that are similar enough such that testing between them is difficult, but such that at the same time the conditional functionals associated to these distributions are far apart.

\section{Conclusion}

In the present paper we have established lower and upper bounds on maximal expected regret in a functional sequential assignment problem with covariates. Our results show that the optimality theory developed in \cite{rigollet2010nonparametric} can be generalized to a large class of functionals of the conditional potential outcome distributions beyond the conditional mean.

\appendix
\appendixpage

\section{Auxiliary results}

We shall use similar notational conventions as discussed in Appendix~A of~\cite{kpv1}. We repeat them here for the convenience of the reader: The (unique) probability measure on the Borel sets of $\R$ corresponding to a cdf~$F$ will be denoted by~$\mu_{F}$, cf., e.g., \cite{folland},~p.35. We employ standard notation and terminology concerning stochastic kernels and their semi-direct products as discussed, e.g., in Appendix~A.3 of \cite{liese}, cf.~in particular their Equation~A.3. The random variables and vectors appearing in the proofs are defined on an underlying probability space~$(\Omega, \mathcal{A}, \P)$ with corresponding expectation~$\E$. This underlying probability space is assumed to be rich enough to support all random variables we work with. A generic element of~$\Omega$ shall be denoted by~$\omega$. For a definition and proofs of elementary properties of the Kullback-Leibler divergence~$\mathsf{KL}(P, Q)$ between two probability measures~$P$ and~$Q$ we refer to~\cite{tsybakov2009introduction}. We use the following general version of a chain rule for Kullback-Leibler divergences. A proof can be found in Appendix~A of \cite{kpv1}.
\begin{lemma}[``Chain rule'' for Kullback-Leibler divergence]\label{lem:CHAIN}
Let~$(\mathcal{X}, \mathfrak{A})$ and~$(\mathcal{Y}, \mathfrak{B})$ be measurable spaces. Suppose that~$\mathfrak{B}$ is countably generated. Let~$\mathsf{A}, \mathsf{B} : \mathcal{B} \times \mathcal{X} \to [0, 1]$ be stochastic kernels, and let~$P$ and~$Q$ be probability measures on~$(\mathcal{X}, \mathfrak{A})$. Then,
\begin{equation}\label{eqn:CHAIN}
\mathsf{KL}(\mathsf{A} \otimes P, \mathsf{B} \otimes Q) = \int_{\mathcal{X}} \mathsf{KL}(\mathsf{A}(\cdot,x), \mathsf{B}(\cdot,x))dP(x) + \mathsf{KL}(P, Q) = \mathsf{KL}(\mathsf{A} \otimes P, \mathsf{B} \otimes P) + \mathsf{KL}(P, Q).
\end{equation}
\end{lemma}
We begin by establishing two auxiliary results that will be useful in the proofs of Theorems~\ref{Thm:Covariate} and~\ref{UBCovariateMargin}. For~$n \in \N$ let~$B_{n,1}, \hdots, B_{n,M}$ be a partition of~$[0, 1]^d$, where every~$B_{n,j}$ is Borel measurable. Given such a partition, for every~$j$ such that~$\P_X(B_{n,j}) > 0$, we shall denote by~$F_{n,j}^{*}$ an element of~$\{F_{n,j}^{i}: i = 1, \hdots, K\}$ (see Equation~\eqref{Fji} for a definition of~$F^i_{n,j}$), such that~$\mathsf{T}(F_{n,j}^{*})=\max_{i \in \mathcal{I}} \mathsf{T}(F_{n,j}^{i})$. Furthermore, we often write~$\pi_{n,t}(X_t)$ instead of~$\pi_{n,t}(X_t,Z_{t-1}, G_t)$ in many places throughout the appendix.

\begin{lemma}\label{ErrorHolder}  Suppose that Assumptions~\ref{as:MAIN} and~\ref{ass:Holder} are satisfied (the latter with~$\gamma \in (0, 1]$ and~$L>0$), and assume that the inclusion in Equation~\eqref{eqn:marginDincls} holds. Let~$B_{n,1}, \hdots, B_{n,M}$ be a partition of~$[0, 1]^d$, where every~$B_{n,j}$ is Borel measurable. As in the statement of Theorem~\ref{Thm:Covariate}, we let~$V_{n,j} = \sup_{x_1, x_2\in B_{n,j}}\enVert[0]{x_1-x_2}$. Then, for every~$i \in \{1, \hdots, K\}$, every~$j \in \{1, \ldots, M \}$ and every pair~$x$ and~$\tilde{x} \in B_{n,j}$, we have
\begin{equation}\label{eqn:ineqsec31}
|\mathsf{T}( F^{i}(\cdot, x))-\mathsf{T}( F^{i}(\cdot, \tilde x))|  \leq C L V_{n,j}^{\gamma} \quad \text{ and } \quad 
|\mathsf{T} \big( F^{\pi^{\star}(x)}(\cdot, x) \big)-\mathsf{T} \big(F^{\pi^{\star}(\tilde x)}(\cdot, \tilde x) \big)|  \leq C L V_{n,j}^{\gamma};
\end{equation}
furthermore, if $\P_X(B_{n,j}) > 0$ holds, then 
\begin{equation}\label{eqn:ineqsec32}
|\mathsf{T}( F_{n,j}^{i})-\mathsf{T}(F^{i}(\cdot,x))|  \leq C L V_{n,j}^{\gamma} \quad \text{ and } \quad |\mathsf{T}( F^{\pi^\star(x)} (\cdot, x))-\mathsf{T}(F_{n,j}^{*})|  \leq C L V_{n,j}^{\gamma}. 
\end{equation}
\end{lemma}
\begin{proof}
Fix~$i$,~$j$,~$x$ and~$\tilde{x}$ as in the statement of the lemma. By Assumption~\ref{ass:Holder}
\begin{equation}\label{eqn:firstinobsB}
||F^{i}(\cdot, x)-F^{i}(\cdot, \tilde x)||_{\infty} \leq L ||x-\tilde{x}||^{\gamma} \leq L V_{n,j}^{\gamma}
\end{equation}
Assumption~\ref{as:MAIN} and~\eqref{eqn:marginDincls} thus imply the first inequality in~\eqref{eqn:ineqsec31}, and the second follows from
\begin{align*}
|\mathsf{T} \big( F^{\pi^{\star}(x)}(\cdot, x) \big)-\mathsf{T} \big(F^{\pi^{\star}(\tilde x)}(\cdot, \tilde x) \big)| &=
|\max_{i \in \mathcal{I}} \mathsf{T}( F^{i}(\cdot, x))-\max_{i \in \mathcal{I}} \mathsf{T}( F^{i}(\cdot, \tilde x))| \\
& \leq \max_{i \in \mathcal{I}} |\mathsf{T}( F^{i}(\cdot, x))-\mathsf{T}( F^{i}(\cdot, \tilde x))| \leq C L V_{n,j}^{\gamma}.
\end{align*}
Next, assume that~$\P_X(B_{n,j}) > 0$. For every~$y \in \R$, from Equation~\eqref{eqn:firstinobsB}, we obtain 
\begin{align*}
|F_{n,j}^i(y)-F^{i}(y, x)|\leq\frac{1}{\P_X(B_{n,j})} \int_{B_{n,j}} |F^i(y, s)-F^i(y, x)| d\P_X(s) \leq LV_{n,j}^{\gamma}.
\end{align*}
The first inequality in~\eqref{eqn:ineqsec32} is now a direct consequence of Assumption~\ref{as:MAIN} and~\eqref{eqn:marginDincls} (noting that~$F_{n,j}^i \in D_{cdf}([a,b])$), and the second inequality follows via
\begin{equation}
|\mathsf{T}( F^{\pi^\star(x)} (\cdot, x))-\mathsf{T}(F_{n,j}^{*})| =\big|\max_{i \in \mathcal{I}}\mathsf{T}( F^{i} (\cdot, x))-\max_{i \in \mathcal{I}}\mathsf{T}(F_{n,j}^{i}) \big| 
\leq \max_{i \in \mathcal{I}} | \mathsf{T}( F^{i} (\cdot, x))-\mathsf{T}(F_{n,j}^{i})|.
\end{equation}
\end{proof}

\begin{lemma}\label{lem:convex}
Suppose Assumption~\ref{as:MAIN} is satisfied and that~$\mathscr{D}$ is convex. Suppose further that~$\P_{Y, X}$ is such that Equation~\eqref{eqn:marginDincls} holds, and that Assumption \ref{ass:Holder} is satisfied. Then, for every Borel set~$B \subseteq [0, 1]^d$ that satisfies~$\P_X(B) > 0$ and every~$i = 1, \hdots, K$, the cdf
\begin{equation}
G_i := \P_X(B)^{-1} \int_B F^i(\cdot, x) d\P_X(x) 
\end{equation}
is an element of the closure of~$\mathscr{D}  \subseteq D_{cdf}([a,b])$ w.r.t.~$\|\cdot\|_{\infty}$.
\end{lemma}

\begin{proof}
Let~$i \in \{1, \hdots, K\}$. We construct a sequence of convex combinations of (finitely many) elements of~$\mathscr{D}$ that converges to~$G_i$ in~$\|\cdot\|_{\infty}$-distance: To this end, let~$B_{m, 1}, \hdots, B_{m,l_m}$ for~$m \in \N$ be a triangular array of partitions of~$[0, 1]^d$ into non-empty Borel subsets, such that the maximal diameter~$v_m := \sup_{i = 1, \hdots, l_m}\sup_{x_1, x_2 \in B_{m,i}} \|x_1 - x_2\| \to 0$ as~$m \to \infty$. For simplicity, define the probability measure~$\P^*$ on the Borel sets of~$\R^d$ by~$\P^*(A) = \P_X(A\cap B)/\P_X(B)$. Write
\begin{equation}
G_i = \int F^i(\cdot, x) d\P^*(x) = \sum_{j = 1}^{l_m} \int_{B_{m,j}}  F^i(\cdot, x) d\P^*(x).
\end{equation}
For every~$m$ and every~$j$, pick an~$x_{m,j} \in B_{m,j}$. Note that~$F^i(\cdot, x_{m,j})\in \mathscr{D}$ by Equation~\eqref{eqn:marginDincls}. From Assumption~\ref{ass:Holder}, we know that for any~$x \in B_{m,j}$ we have~$\|F^i(\cdot, x_{m,j}) - F^i(\cdot, x)\|_{\infty} \leq L\|x_{m,j} - x\|^{\gamma} \leq L v_m^{\gamma}$. Thus,
\begin{equation}
\|G_i - \sum_{j = 1}^{l_m} \P^*(B_{m,j}) F^i(\cdot, x_{m,j})\|_{\infty} \leq \sum_{j = 1}^{l_m} \int_{B_{m,j}}  \|F^i(\cdot, x) - F^i(\cdot, x_{m,j})\|_{\infty} d \P^*(x) \leq L v_m^{\gamma} \to 0.
\end{equation}
\end{proof}

\section{Proofs of results in Section~\ref{sec:setupX}}

\subsection{Proof of Theorem~\ref{thm:LN_UC}}

Because Assumption~\ref{ass:Holder} (for any~$\gamma \in (0, 1]$ and any~$L>0$) implies the assumption in Equation~\eqref{eqn:marginUEC}, the statement follows immediately from the lower bound in Equation~\eqref{eqn:suplow} in  Theorem~\ref{LBCovMC} upon letting~$\gamma \to 0$.

\subsection{Proof of Theorem~\ref{LBIgnoreCov}}

If~$\min(\P_X(A_1), \P_X(A_2)) = 0$, then the statement in the theorem trivially holds. Hence, assume that~$p:= \min(\P_X(A_1), \P_X(A_2)) > 0$. Let~$n \in \N$ and let~$\pi$ be a policy that ignores covariates, i.e., as described before Theorem~\ref{LBIgnoreCov}. We write~$\pi_{n,t} = \pi_t$. Fix a randomization measure~$\P_G$.

As a preparation, for every~$m \in \N$, define
\begin{align*}
&A_{1,m} := \{x \in [0, 1]^d: \mathsf{T}(F^1(\cdot, x)) > m^{-1} + \mathsf{T}(F^2(\cdot, x))\}, \\
&A_{2,m} := \{x \in [0, 1]^d: \mathsf{T}(F^1(\cdot, x))  +m^{-1} < \mathsf{T}(F^2(\cdot, x))\}.
\end{align*}
The sets~$A_1, A_2$ and~$A_{1,m}, A_{2,m}$ for~$m \in \N$ are Borel measurable, because Assumptions~\ref{as:MAIN} and~\ref{ass:Holder} together with Equation~\eqref{eqn:marginDincls} imply the continuity of~$x \mapsto \mathsf{T}(F^{i}(\cdot, x))$ for~$i = 1,2$. Note that~$A_{i,m} \subseteq A_{i,m+1}$ and~$\bigcup_{m \in \N} A_{i,m} = A_i$ hold for~$i = 1,2$. Hence, as~$m \to \infty$,~$\mathbb{P}_X(A_{i,m}) \to \mathbb{P}_X(A_{i})$ for~$i = 1,2$. Because of~$p>0$, we can conclude the existence of an~$\bar{m} \in \N$ such that~$p_{\bar{m}} := \min(\P_X(A_{1,\bar{m}}), \P_X(A_{2,\bar{m}})) > p/2$. To prove the inequality in Equation~\eqref{MinMaxIneq}, note that by definition, and since~$\pi$ is a policy that does not depend on covariates, i.e., the~$t$-th assignment only depends on the previously observed outcomes and randomizations,~$\tilde{Z}_{t-1}$ and a novel randomization~$G_t$, we have (cf.~the discussion and notation discussed right before the statement of Theorem~\ref{LBIgnoreCov}) that
\begin{equation*}
R_n(\pi)=\sum_{t=1}^n \big|\mathsf{T}\big(F^{1}(\cdot, X_t)\big)-\mathsf{T}\big(F^{2}(\cdot, X_t)\big) \big|\mathds{1}_{\{ \pi^{\star}(X_t) \neq \tilde{\pi}_t(\tilde{Z}_{t-1}, G_t) \}}.
\end{equation*}
Note furthermore that 
\begin{align}\label{eqn:setincllow}
&\left[\{X_t \in A_{1,{\bar{m}}}\} \cap \{\tilde{\pi}_t(\tilde{Z}_{t-1}, G_t) \neq 1\}\right] \cup \left[\{X_t \in A_{2,{\bar{m}}}\} \cap \{\tilde{\pi}_t(\tilde{Z}_{t-1}, G_t) \neq 2\}\right] \\
\subseteq ~& \{ \pi^{\star}(X_t) \neq \tilde{\pi}_t(\tilde{Z}_{t-1}, G_t) \}.
\end{align}
where the union in the first line is a disjoint union. Hence, 
\begin{equation*}
R_n(\pi) \geq {\bar{m}}^{-1} \sum_{t = 1}^n \left( \mathds{1}_{A_{1,{\bar{m}}}}(X_t) \mathds{1}_{\{\tilde{\pi}_t(\tilde{Z}_{t-1}, G_t) \neq 1\}} + \mathds{1}_{A_{2,{\bar{m}}}}(X_t) \mathds{1}_{\{\tilde{\pi}_t(\tilde{Z}_{t-1}, G_t) \neq 2\}} \right).
\end{equation*}
Since~$X_t$ is independent of~$\tilde{Z}_{t-1}$ and~$G_t$, the law of iterated expectations implies~$\mathbb{E}(R_n(\pi)) \geq  np/(2\bar{m})$. 

\section{Proofs of results in Section~\ref{sec:FUCBX}}

\subsection{Proof of Theorem~\ref{Thm:Covariate}}

Fix~$n \in \N$ and let~$(Y_t, X_t) \sim \P_{Y, X}$ for~$t = 1, \hdots, n$, where~$\P_{Y, X}$ satisfies Equation~\eqref{eqn:marginDincls}, and Assumption~\ref{ass:Holder} with~$L$ and~$\gamma$. Because~$n$ is fixed, we abbreviate~$B_{n,j} = B_j$,~$V_{n,j} = V_j$,~$M(n) = M$, and denote~$\bar{\pi}_{n,t} = \bar{\pi}_t$. First, we decompose~$R_n(\bar{\pi})=\sum_{j=1}^M\tilde{R}_j(\bar{\pi})$, where
\begin{align}\label{eqn:decompThmcovariate}
\tilde{R}_j(\bar{\pi})
:=\sum_{t=1}^n \big[ \mathsf{T} \big(F^{\pi^\star(X_t)} (\cdot, X_t) \big)-\mathsf{T}\big(F^{\bar{\pi}_t(X_t)} (\cdot, X_t) \big) \big] \mathds{1}_{\{ X_{t}\in B_j \} },
\end{align}
where, as often done in the present section, we dropped the argument~$Z_{t-1}$ from~$\bar{\pi}_t$. Note furthermore that the policy does not rely on an external randomization~$G_t$, which is therefore suppressed in the notation as well. 

Note first that the boundedness of~$\mathsf{T}$ on~$\mathscr{D}$ (cf.~Assumption~\ref{as:MAIN}) implies~$\E(\tilde{R}_j(\bar{\pi})) = 0$ for every~$j$ such that~$\P_X(B_j) = 0$. Hence, we now fix an index~$j \in \{1, \hdots, M\}$, such that~$\P_X(B_j) > 0$. Then, recalling the definition of~$F_{n,j}^{i}$ in Equation~\eqref{Fji}, which we here abbreviate as~$F_{j}^{i}$, each summand in~\eqref{eqn:decompThmcovariate} can be written as
\begin{equation}
\left[\mathsf{T} (F^{\pi^\star(X_t)} (\cdot, X_t) )- \mathsf{T}(F_j^*)
+ \mathsf{T}(F_j^*) - \mathsf{T}(F_j^{\bar{\pi}_t(X_t)})
+ \mathsf{T}(F_j^{\bar{\pi}_t(X_t)})
- \mathsf{T}(F^{\bar{\pi}_t(X_t)} (\cdot, X_t) )	
\right]  \mathds{1}_{\{ X_{t}\in B_j \} },
\end{equation}
which, by Lemma~\ref{ErrorHolder}, is not greater than~$\mathsf{T}(F_j^*) - \mathsf{T}(F_j^{\bar{\pi}_t(X_t)})+2 C L V_j^{\gamma}$, and where~$F_{j}^*$ was defined just before~Lemma~\ref{ErrorHolder}. Therefore, we obtain 
\begin{align}\label{TildeR1}
\tilde{R}_j(\bar{\pi}) \leq\sum_{t=1}^n \left[\mathsf{T}\big(F_j^{*}\big)-\mathsf{T}\big(F_j^{\bar{\pi}_t(X_t)} \big) \right] \mathds{1}_{ \{X_{t}\in B_j \} }+2CLV_j^\gamma 
\sum_{t=1}^n \mathds{1}_{ \{X_{t}\in B_j \}}.
\end{align}	
Obviously,~$\mathbb{E}(\sum_{t=1}^n \mathds{1}_{ \{X_{t}\in B_j \}}) = n\P_X(B_j)$. Hence, to prove the theorem, it remains to show that for~$c = c(\beta, C)$ as defined in the statement of the theorem it holds that
\begin{equation}\label{eqn:claimcovband}
\mathbb{E}\left(\sum_{t=1}^n \Big[\mathsf{T}\big(F_j^{*}\big)-\mathsf{T}\big(F_j^{\bar{\pi}_t(X_t)} \big) \Big] \mathds{1}_{ \{X_{t}\in B_j \} } \right) \leq c \sqrt{K n \P_X(B_j)  \overline{\log}(n \P_X(B_j) )}.
\end{equation}

To this end we will use a conditioning argument in combination with Theorem~4.1 in~\cite{kpv1}. Define for every~$v = (v_1, \hdots, v_n) \in \{0, 1\}^n$ the event
\begin{equation}
\Omega(v) := \{\omega: \mathds{1}_{ \{X_{t}\in B_j \} }(\omega) = v_t \text{ for } t = 1, \hdots, n\},
\end{equation}
and denote~$f := \sum_{t=1}^n [\mathsf{T}\big(F_j^{*}\big)-\mathsf{T}\big(F_j^{\bar{\pi}_t(X_t)} \big) ] \mathds{1}_{ \{X_{t}\in B_j \} }$. Then,
\begin{equation}\label{eqn:toncovexpand}
\mathbb{E}(f) = \sum_{v \in \{0, 1\}^n
} \mathbb{E}(\mathds{1}_{\Omega(v)} f) = \sum_{v \in \{0, 1\}^n} \mathbb{P}({\Omega(v)})\mathbb{E}( f | {\Omega(v)}),
\end{equation}
where (as usual) we define
\begin{equation}
\mathbb{E}( f | {\Omega(v)}) :=
\begin{cases}
\mathbb{P}^{-1}(\Omega(v)) \mathbb{E}(\mathds{1}_{\Omega(v)} f) & \text{ if } \mathbb{P}(\Omega(v)) > 0, \\
0 & \text{ else. }
\end{cases}
\end{equation}
Fix~$v \neq 0$. Denote the elements of~$\{s:v_s = 1\}$ by~$t_1, \hdots, t_{\bar{m}}$, ordered from smallest to largest. On the event~$\Omega(v)$, i.e., for every~$\omega \in \Omega(v)$, we can use the definition of~$\bar{\pi}$ (cf.~the description of the F-UCB policy with covariates of display Policy~\ref{poly:FUCBC}) to rewrite  
\begin{equation}
f  = \sum_{s = 1}^{\bar{m}} \left[ \mathsf{T}(F_j^*) - \mathsf{T}\left(F_j^{\hat{\pi}_{s}(W^{s-1})}\right) \right],
\end{equation}
where~$\hat{\pi}$ is the F-UCB policy from~\cite{kpv1}, and where~$W^s$ is defined recursively via~$W^{s} = (Y_{\hat{\pi}_{s-1}(W^{s-1}), t_s}, W^{s-1})$ with $W^0$ the empty vector (cf.~also the discussion before our Policy~\ref{poly:FUCBC}). Hence, for~$\omega \in \Omega(v)$,~$f$ is a function of~$(Y_{t_1}, \hdots, Y_{t_{\bar{m}}})$, i.e.,~$f = H(Y_{t_1}, \hdots, Y_{t_{\bar{m}}})$, say. We conclude that 
\begin{equation}
\mathbb{E}(f|\Omega(v)) =
\mathbb{E}\left( H(Y_{t_1}, \hdots, Y_{t_{\bar{m}}}) | \Omega(v) \right) = \mathbb{E}^v(H(Y_{t_1}, \hdots, Y_{t_{\bar{m}}})),
\end{equation}
where the probability measure~$\mathbb{P}^v$ corresponding to~$\mathbb{E}^v$ is defined as the~$\mathbb{P}$-measure with density~$\mathbb{P}^{-1}(\Omega(v)) \mathds{1}_{\Omega(v)}$. Note that for~$A_i \in \mathcal{B}(\R^K)$ for~$i = 1, \hdots, \bar{m}$, we have that~$\mathbb{P}^v(Y_{t_1} \in A_1, \hdots, Y_{t_{\bar{m}}} \in A_{\bar{m}}  )$ equals 
\begin{align}
\mathbb{P}^{-1}(\Omega(v)) \mathbb{P}\left( Y_{t_1} \in A_1, \hdots, Y_{t_{\bar{m}}} \in A_{\bar{m}}, \Omega(v) \right) 
&= \prod_{s = 1}^{\bar{m}} \frac{\mathbb{P}(Y_{t_s} \in A_s, X_{t_s} \in B_j)}{ \mathbb{P}(X_{t_s} \in B_j)} \\ &= \prod_{s = 1}^{\bar{m}} \mathbb{P}(Y_{t_s} \in A_s | \{X_{t_s} \in B_j\}).
\end{align}
Hence, the image measure~$\mathbb{P}^v \circ (Y_{t_1}, \hdots, Y_{t_{\bar{m}}})$ is the~$\bar{m}$-fold product of~$\mathbb{Q}(\cdot) := \mathbb{P}(Y_{1} \in . | \{X_{1} \in B_j\})$. For i.i.d.~random~$K$-vectors~$Y^*_1, \hdots, Y^*_{\bar{m}}$, say, each with distribution~$\mathbb{Q}$, it hence follows from the definition of~$H$ that 
\begin{equation}
\mathbb{E}(H(Y_{t_1}, \hdots, Y_{t_m})|\Omega(v)) = \mathbb{E}(H(Y^*_1, \hdots, Y^*_{\bar{m}})) = \mathbb{E}\left( 
\sum_{s = 1}^{\bar{m}} \left[ \mathsf{T}(F_j^*) - \mathsf{T}\left(F_j^{\hat{\pi}_{s}(Z^*_{s - 1})}\right) \right] 
\right)
\end{equation}
where~$Z^*_{s} = (Y^*_{\hat{\pi}_{s}(Z^*_{s-1}), s}, \hdots, Z^*_{s-1})$ (and where~$Z^*_0$ is the empty vector). The~$r$-th marginal of~$\mathbb{Q}$ has cdf~$F^r_j$, which by Lemma~\ref{lem:convex} is an element of the closure of~$\mathscr{D} \subseteq D_{cdf}([a,b])$ w.r.t.~$\|\cdot\|_{\infty}$, which we here denote as~$\mathrm{cl}(\mathscr{D})$. Therefore, it now follows from Theorem~4.1 in~\cite{kpv1}, applied with~$\mathrm{cl}(\mathscr{D})$ (cf.~their Remark 2.4) and with~``$n = \bar{m}$,'' that the quantity in the previous display, and thus~$\mathbb{E}(f|\Omega(v))$, is not greater than~$c \sqrt{K \bar{m} \overline{\log}(\bar{m})}$. From~\eqref{eqn:toncovexpand} (noting that~$f$ vanishes on~$\Omega(0)$) we see that
\begin{equation}
\mathbb{E}(f) \leq c\sum_{v \in \{0, 1\}^n} \mathbb{P}(\Omega(v))  \sqrt{K \bar{m} \overline{\log}(\bar{m})}.
\end{equation}
Recall, that~$\bar{m} = \sum_{s = 1}^n v_s$. Hence, we can interpret~$\bar{m}$ as a random variable on the set~$\{0, 1\}^n$, equipped with the probability mass function~$p(v) = \mathbb{P}(\Omega(v))$. Obviously, this random variable is Bernoulli-distributed with success probability~$\P_X(B_j)$ and ``sample size'' $n$. Thus its expectation is~$n\P_X(B_j)$. It remains to observe that the function~$h$ defined via~$x \mapsto (K x \overline{\log}(x))^{0.5}$ is concave on~$[0, \infty)$, allowing us to apply Jensen's inequality to upper bound the right hand side in the previous display by~$c h(n \P_X(B_j))$, which establishes the statement in Equation~\eqref{eqn:claimcovband}.

\subsection{Proof of Corollary~\ref{Cor:SimpleBins}}\label{sec:proofsimplebins}

Fix~$n \in \N$, and let~$(Y_t, X_t) \sim \P_{Y, X}$ for~$t = 1, \hdots, n$, where~$\P_{Y, X}$ satisfies Equation~\eqref{eqn:marginDincls}, Assumption~\ref{ass:DensityX} with~$\underline{c}$ and~$\overline{c}$, and Assumption~\ref{ass:Holder} with~$L$ and~$\gamma$. We shall apply Theorem~\ref{Thm:Covariate} to get an upper bound on~$\E[R_n(\bar{\pi})]$. The specific partition results in~$M(n)=P^d$ and~$V_{n,j}=\sqrt{d} P^{-1}$, where~$P = \lceil n^{1/(2\gamma + d)} \rceil$. Furthermore, from Assumption~\ref{ass:DensityX}, we obtain~$\P_X(B_{n,j}) \leq \overline{c} P^{-d}$. Therefore, Equation~\eqref{eq:regret_cov_gen} implies the upper bound 
\begin{equation}
\E[R_n(\bar{\pi})] \leq 
c(\beta, C) \sqrt{ K n \bar{c} P^{d} \overline{\log} (n \bar{c} P^{-d}) } + 2 C L (\sqrt{d} P^{-1})^{\gamma} n \bar{c},
\end{equation}
which (using monotonicity of~$\overline{\log}$, and~$\overline{\log}(xy) \leq \overline{\log}(x) + \overline{\log}(y)$ for positive~$x$ and~$y$) is bounded from above by 
\begin{align}
c(\beta, C) \sqrt{ K \bar{c}(1+\overline{\log}(\bar{c})) \overline{\log} (n) n P^{d}  } + 2 C L d^{\gamma/2}  \bar{c} nP^{-\gamma}  &\leq c^*\left( \sqrt{K \overline{\log}(n) n P^d} + nP^{-\gamma} \right) \\ &\leq c^* \sqrt{K \overline{\log}(n)} \left( \sqrt{n P^d} + nP^{-\gamma} \right) ,
\end{align}
where~$c^* := \max [c(\beta, C) (\bar{c}(1+\overline{\log}(\bar{c})))^{1/2}, 2 C L d^{\gamma/2}  \bar{c}]$. From~$P^{-\gamma} \leq n^{-\gamma/(2\gamma + d)}$ and~$P^d \leq 2^d n^{d/(2\gamma + d)}$, we obtain the bound
\begin{equation}
\E[R_n(\bar{\pi})] \leq (2^{d/2} + 1) c^* \sqrt{K \overline{\log}(n)} n^{1-\frac{\gamma}{2\gamma + d}},
\end{equation}
which proves the theorem.

\subsection{Proof of Theorem~\ref{thm:LB_NoMargin}}

The statement follows from the first lower bound established in Theorem~\ref{LBCovMC}, upon setting~$\alpha = \alpha(\varepsilon) = (2\gamma + d)\varepsilon/\gamma$ there; note that~$\alpha(\varepsilon)$ is an element of~$(0, 1)$, because~$\eps \in (0, \gamma/(2\gamma + d))$ holds by construction. 

\subsection{Proof of Theorem~\ref{UBCovariateMargin}}

Define~$c_1 := 4 C L d^{\gamma/2}+1$. Recall that~$P=\lceil n^{1/(2 \gamma +d)}  \rceil$. Note first that it suffices to establish the inequality in Equation~\eqref{eqn:UBCovariateMargin} for all~$n$ large enough ($n \geq n_0$, say), such that~$c_1 P^{-\gamma} \leq 1$ holds (this will allow us to apply Assumption~\ref{ass:Margin} with~$\delta=c_1 P^{-\gamma}$ in the arguments below). To see this, note that, by Assumption~\ref{as:MAIN}, for all~$n < n_0$ it holds (for all random vectors as in the statement of the theorem) that~$\mathbb{E}[R_n(\pi)] \leq C n_0$. Hence, once the claimed inequality in the theorem has been established for all~$n \geq n_0$, the constant~$c$ in the statement of Theorem~\ref{UBCovariateMargin} can be chosen large enough to deal with the initial terms smaller than~$n_0$. Hence, fix~$n \geq n_0$. Because~$n$ is fixed, we abbreviate~$B_{n,j} = B_j$,~$V_{n,j} = V_j = \sqrt{d} P^{-1}$, and denote~$\bar{\pi}_{n,t} = \bar{\pi}_t$. 

Let~$(Y_t, X_t) \sim \P_{Y, X}$ for~$t = 1, \hdots, n$, where~$\P_{Y, X}$ satisfies Equation~\eqref{eqn:marginDincls}, Assumption~\ref{ass:DensityX} with~$\underline{c}$ and~$\overline{c}$, Assumption~\ref{ass:Holder} with~$L$ and~$\gamma$, and Assumption~\ref{ass:Margin} with~$\alpha \in (0, 1)$ and~$C_0 > 0$. We establish~$\E[R_n(\bar{\pi})] \leq cK \overline{\log}(n)n^{1-\frac{\gamma(1+\alpha)}{2\gamma + 2}}$ for a constant that depends on the quantities indicated in the statement of the theorem in five steps:

\medskip

\noindent
\textbf{Step~1: Decomposition of bins into different types.} To obtain the desired upper bound, we shall treat three types of bins separately. An analogous division of bins was also used in \cite{perchet2013multi} to establish the properties of their successive elimination algorithm in a classic bandit problem targeting the distribution with the highest (conditional) mean. The bins are split into
\begin{equation}\label{eqn:wsibins}
\begin{aligned}
\mathcal{J}&:=\left\{ j \in \{1, \ldots, P^d \}: \exists~\bar{x} \in B_j, \mathsf{T}(F^{\pi^{\star}(\bar{x})}(\cdot,\bar{x}))-
\mathsf{T}(F^{\pi^{\sharp}(\bar{x})}(\cdot,\bar{x}) )>c_1 P^{-\gamma} \right\}, \\
\mathcal{J}_s&:=\left\{ j \in \{1, \ldots, P^d \}: \exists~\bar{x} \in B_j, \mathsf{T}(F^{\pi^{\star}(\bar{x})}(\cdot,\bar{x}))=
\mathsf{T}(F^{\pi^{\sharp}(\bar{x})}(\cdot,\bar{x}) ) \right\}, \\
\mathcal{J}_{w}&:=\left\{j \in \{1, \ldots, P^d \}: 0<\mathsf{T}(F^{\pi^{\star}(x)}(\cdot,x))-
\mathsf{T}(F^{\pi^{\sharp}(x)}(\cdot,x) )\leq c_1 P^{-\gamma} \mbox{ for all } x \in B_j \right\}.
\end{aligned}
\end{equation}
The bins corresponding to indices in~$\mathcal{J}$,~$\mathcal{J}_s$, and~$\mathcal{J}_{w}$ will be referred to as ``well-behaved,'' ``strongly ill-behaved'' and ``weakly ill-behaved'' bins, respectively. Note that~$\mathcal{J}_w$ and~$\mathcal{J} \cup \mathcal{J}_s$ are clearly disjoint. That~$\mathcal{J}$ and~$\mathcal{J}_s$ are disjoint is shown in Step~2 below. Hence, the sets of bins corresponding to indices in~$\mathcal{J}$,~$\mathcal{J}_s$,~$\mathcal{J}_{w}$ constitute a partition of the set of all~$P^d$ bins~$B_j$, and we can thus write
\begin{equation}\label{eqn:Regretthreetypes}
\mathbb{E}(R_n(\bar{\pi})) = \sum_{j \in \mathcal{J}_s} \mathbb{E}(\tilde{R}_j(\bar\pi)) + \sum_{j \in \mathcal{J}_w}\mathbb{E}(\tilde{R}_j(\bar\pi)) + \sum_{j \in \mathcal{J}}\mathbb{E}(\tilde{R}_j(\bar\pi)),
\end{equation}
where, as in Equation~\eqref{eqn:decompThmcovariate}, we define
\begin{equation}\label{RegretBinj}
\tilde{R}_j(\bar\pi):=\sum_{t=1}^n \Big[\mathsf{T}\big(F^{\pi^{\star}(X_t)}(\cdot, X_t)\big)-\mathsf{T}\big(F^{\bar{\pi}_t(X_t)}(\cdot, X_t)\big) \Big]\mathds{1}_{\{ X_t \in B_j   \}}.
\end{equation}

\bigskip

\noindent
\textbf{Step~2: Strongly ill-behaved bins.}  For every~$j \in \mathcal{J}_s$, by definition, there exists a~$\bar{x} \in B_j$ such that 
$\mathsf{T}\big(F^{\pi^{\star}(\bar{x})}(\cdot,\bar{x})\big)=\mathsf{T}\big(F^{\pi^{\sharp}(\bar{x})}(\cdot,\bar{x}) \big)$. From the definition of~$\pi^{\sharp}$ it thus follows that~$\mathsf{T}\big(F^{\pi^{\star}(\bar{x})}(\cdot,\bar{x})\big)=\mathsf{T}\big(F^{i}(\cdot,\bar{x}) \big)$ for every~$i \in \mathcal{I}$. Therefore, for every~$x \in B_j$ and every~$i \in \mathcal{I},$ Lemma~\ref{ErrorHolder} yields
\begin{align}\label{eqn:use1sw}
\mathsf{T}(F^{\pi^{\star}(x)}(\cdot,x))-\mathsf{T}(F^{i}(\cdot,x))&=\mathsf{T}(F^{\pi^{\star}(x)}(\cdot,x))-
\mathsf{T}(F^{i}(\cdot,x))-[\mathsf{T}(F^{\pi^{\star}(\bar{x})}(\cdot,\bar{x}))-\mathsf{T}(F^{i}(\cdot,\bar{x}))] \\
& \leq 2 C L d^{\gamma/2} P^{-\gamma} \leq c_1 P^{-\gamma}.
\end{align}
First of all, this shows that~$\mathcal{J}$ and~$\mathcal{J}_s$ are disjoint. Furthermore, from Equations~\eqref{RegretBinj} and~\eqref{eqn:use1sw}, we obtain 
\begin{align}
\sum_{j \in \mathcal{J}_s} \tilde{R}_j(\bar{\pi}) &\leq c_1 P^{-\gamma} \sum_{j \in \mathcal{J}_s} \sum_{t=1}^n \mathds{1}_{\{ X_t \in B_j  \} } \mathds{1}_{\{ 0<\mathsf{T}(F^{\pi^{\star}( X_t)}(\cdot, X_t))-\mathsf{T}(F^{\pi^{\sharp}( X_t)}(\cdot, X_t) ) \}} \\
&\leq 
c_1 P^{-\gamma} \sum_{t=1}^n \mathds{1}_{\{ 0 < \mathsf{T}(F^{\pi^{\star}(X_t)}(\cdot,X_t))-\mathsf{T}(F^{\pi^{\sharp}(X_t)}(\cdot,X_t)) \leq c_1 P^{-\gamma}  \} }.
\end{align}
From Condition~\ref{ass:Margin} we hence obtain:
\begin{equation}
\begin{aligned}\label{RegStrIllBeh} 
\sum_{j \in \mathcal{J}_s} \mathbb{E}[\tilde{R}_j(\bar{\pi})] &\leq 
c_1  n P^{-\gamma}  \P_X \big(0<\mathsf{T}\big(F^{\pi^{\star}(X)}(\cdot,X)\big)-\mathsf{T}\big(F^{\pi^{\sharp}(X)}(\cdot,X) \leq c_1 P^{-\gamma} \big) \\
&\leq C_0 c_1^{1+\alpha} n P^{-\gamma(1+\alpha)}.
\end{aligned}
\end{equation}

\noindent
\textbf{Step~3: Weakly ill-behaved bins.} Since~$\{X_t \in B_j\}$ for~$j \in \mathcal{J}_w$  are disjoint subsets of
\begin{equation*}
\{ 0<\mathsf{T}(F^{\pi^{\star}(X_t)}(\cdot,X_t))-\mathsf{T}(F^{\pi^{\sharp}(X_t)}(\cdot,X_t)) \leq c_1 P^{-\gamma} \},
\end{equation*}
we obtain from Condition~\ref{ass:Margin}, recall that~$\mathbb{P}(X_t \in B_j) \geq \frac{\underline{c}}{P^d}$, that
\begin{align}
|\mathcal{J}_w| \frac{\underline{c}}{P^d} \leq \sum_{j \in \mathcal{J}_w } \mathbb{P}(X_t \in B_j) 
&\leq \P \big(0<\mathsf{T}\big(F^{\pi^{\star}(X_t)}(\cdot,X_t) \big)-\mathsf{T} \big(F^{\pi^{\sharp}(X_t)}(\cdot,X_t) \big)  
\leq c_1 P^{-\gamma} \big) 
\\ &\leq C_0 c_1^{\alpha} P^{-\gamma \alpha},
\end{align}
which yields~$|\mathcal{J}_w| \leq (C_0 c_1^{\alpha}/\underline{c}) P^{d-\gamma \alpha}.$ Using~\eqref{TildeR1} and~\eqref{eqn:claimcovband} with~$V_j=\sqrt{d} P^{-1}$ and 
$\P_X(B_j) \leq \bar{c} P^{-d}$, we obtain (by similar arguments as in Section~\ref{sec:proofsimplebins})
\begin{align}\label{BoundRegretj}
\mathbb{E}[\tilde{R}_j(\bar{\pi})] \leq
c' \left(\sqrt{K n \overline{\log}(n)} P^{-d/2}+n P^{-\gamma-d}\right),
\end{align}
where~$c'$ depends on~$d, L, \gamma, \bar{c}, C, \beta$, but \emph{not} on~$n$. Combining~\eqref{BoundRegretj} with~$|\mathcal{J}_w| \leq (C_0 c_1^{\alpha}/\underline{c}) P^{d-\gamma \alpha}$ leads to
\begin{align}\label{RegWeakIllBeh}
\sum_{j \in \mathcal{J}_w} \mathbb{E}[\tilde{R}_j(\bar{\pi})] \leq c'' \big(\sqrt{K n \overline{\log}(n)} P^{d/2-\gamma \alpha}+n P^{-\gamma(1+\alpha)} \big),
\end{align}
where~$c''$ depends on~$d, L, \gamma, \underline{c}, \bar{c}, C, C_0, \alpha, \beta$, but \emph{not} on~$n$.

\medskip

\noindent
\textbf{Step~4: Well-behaved bins.} 
For every~$j \in \mathcal{J}$ let~$x_j \in B_j$ be such that
\begin{equation}\label{eqn:Tdiffstarsharp}
\mathsf{T}(F^{\pi^{\star}(x_j)}(\cdot,x_j))-
\mathsf{T}(F^{\pi^{\sharp}(x_j)}(\cdot,x_j) )>c_1 P^{-\gamma}.
\end{equation}
Next, define the following sets of indices (``corresponding to the optimal and suboptimal treatments given~$x_j$''):
\begin{align*}
I_j^{\star}&:=\{ i \in \mathcal{I}: \mathsf{T}\big(F^{\pi^{\star}(x_j)}(\cdot,x_j) \big)=\mathsf{T}(F^{i}(\cdot,x_j)) \},  \\
I_j^{0}&:=\{ i \in \mathcal{I}: \mathsf{T}\big(F^{\pi^{\star}(x_j)}(\cdot,x_j) \big)-\mathsf{T}(F^{i}(\cdot,x_j))>c_1 P^{-\gamma} \}.
\end{align*}
Clearly~$\pi^{\star}(x_j) \in I_j^{\star}$ and~$\pi^{\sharp}(x_j) \in I_j^{0}$ (cf.~\eqref{eqn:Tdiffstarsharp}). Hence~$I_j^{\star}$ and~$I_j^{0}$ define a nontrivial partition of~$\mathcal{I}$. For every~$j \in \mathcal{J}$ we can thus decompose~$\tilde{R}_j(\bar\pi)$ defined in Equation~\eqref{RegretBinj} as the sum of
\begin{equation}
\begin{aligned}
&\tilde{R}_{j,I_j^\star}(\bar{\pi}) :=\sum_{i\in I_j^\star}\sum_{t=1}^n \left[\mathsf{T}\big(F^{\pi^{\star}(X_t)}(\cdot, X_t)\big)-\mathsf{T}\big(F^{i}(\cdot, X_t)\big) \right]\mathds{1}_{\{ X_t \in B_j   \}}\mathds{1}_{\cbr[0]{\bar{\pi}_t(X_t)=i}},\\
&\tilde{R}_{j,I_j^0}(\bar{\pi}) := 
\sum_{i\in I_j^0}\sum_{t=1}^n \left[\mathsf{T}\big(F^{\pi^{\star}(X_t)}(\cdot, X_t)\big)-\mathsf{T}\big(F^{i}(\cdot, X_t)\big) \right]\mathds{1}_{\{ X_t \in B_j   \}}\mathds{1}_{\cbr[0]{\bar{\pi}_t(X_t)=i}}.
\end{aligned}
\end{equation}

\noindent
\textbf{Step~4a: A bound for~$\mathbb{E}(\tilde{R}_{j,I_j^\star}(\bar{\pi}))$.}
For any~$i \in I_j^{\star}$ and every~$x \in B_j$ satisfying~$\mathsf{T}(F^{\pi^{\star}(x)}(\cdot,x)) \neq \mathsf{T}(F^{i}(\cdot,x))$, 
the triangle inequality, the definition of~$\pi^{\sharp}$, and Lemma~\ref{ErrorHolder} yield
\begin{align*}
0&<\mathsf{T}(F^{\pi^{\star}(x)}(\cdot,x))-\mathsf{T}(F^{\pi^{\sharp}(x)}(\cdot,x)) \\
&\leq \mathsf{T}(F^{\pi^{\star}(x)}(\cdot,x))-\mathsf{T}(F^{i}(\cdot,x))\\
&=\mathsf{T}(F^{\pi^{\star}(x)}(\cdot,x))-\mathsf{T}(F^{\pi^{\star}(x_j)}(\cdot,x_j))+\mathsf{T}(F^{i}(\cdot,x_j))-\mathsf{T}(F^{i}(\cdot,x)) \leq 2 C L d^{\gamma/2} P^{-\gamma} \leq c_1 P^{-\gamma},
\end{align*}
the last inequality following from~$c_1 = 4 C L d^{\gamma/2}+1$. But this means (applying the inequality chain in the previous display twice) that for any~$i \in I_j^{\star}$ and every~$x \in B_j$ 
\begin{equation}\label{I0Inequality}
\mathsf{T}(F^{\pi^{\star}(x)}(\cdot,x))-\mathsf{T}(F^{i}(\cdot,x)) \leq c_1 P^{-\gamma} \mathds{1}_{ \{v: 0<\mathsf{T}(F^{\pi^{\star}(v)}(\cdot,v))-\mathsf{T}(F^{\pi^{\sharp}(v)}(\cdot,v)) \leq c_1 P^{-\gamma}  \} } (x).
\end{equation}
We deduce
\begin{align}\label{eqn:4abound}
\E[\tilde{R}_{j,I_j^\star}(\bar{\pi})]
&\leq
\E\sum_{t=1}^n c_1  P^{-\gamma } \mathds{1}_{  \{  0< \mathsf{T}(F^{\pi^{\star}(X_t)}(\cdot,X_t))-\mathsf{T}(F^{\pi^{\sharp}(X_t)}(\cdot,X_t)) \leq c_1 P^{-\gamma} \}   } 
\mathds{1}_{\{ X_t \in B_j   \}}  \leq n c_1 P^{-\gamma} q_j, 
\end{align}
where~$q_j:=\P(0< \mathsf{T}(F^{\pi^{\star}(X_t)}(\cdot,X_t))-\mathsf{T}(F^{\pi^{\sharp}(X_t)}(\cdot,X_t)) \leq c_1 P^{-\gamma}, X_t\in B_j)$, which is independent of~$t$ due to the~$X_t$ being identically distributed.

\medskip

\noindent
\textbf{Step~4b: A bound for~$\mathbb{E}(\tilde{R}_{j,I_j^0}(\bar{\pi}))$.}
By Lemma~\ref{ErrorHolder}, noting that~$\P_X(B_j) > \underline{c}P^{-d} > 0$, for every~$x \in B_j$ and every~$i \in I_j^0$ we have (abbreviating~$F_{n,j}^i$ by~$F_j^i$)
\begin{equation}\label{eqn:step4bfirst}
\mathsf{T}(F^{\pi^{\star}(x)}(\cdot, x))-\mathsf{T}(F^{i}(\cdot, x)) \leq
\left[\mathsf{T}(F^{\ast}_j)-\mathsf{T}(F^{i}_j )\right] + c_1 P^{-\gamma},
\end{equation}
from which it follows that
\begin{equation}\label{eqn:perarmregr}
\begin{aligned}
\E[\tilde{R}_{j,I_j^0}(\bar{\pi})] &\leq
\sum_{i\in I_j^0}\Delta^i_j \E S(i, n, j)
+c_1 P^{-\gamma}\sum_{i\in I_j^0} \E S(i, n, j),
\end{aligned}
\end{equation}
where, for every~$i \in I_j^0$, we let~$S(i, n, j) := \sum_{t=1}^n \mathds{1}_{\{ X_t \in B_j\}}\mathds{1}_{\cbr[0]{\bar{\pi}_t(X_t)=i}}$ and~$\Delta_{j}^i:=\mathsf{T} (F_j^{\ast})-\mathsf{T} (F_j^{i})$. We now claim that (this claim will be verified before moving to Step~4c below)
\begin{equation}\label{eqn:yaclaim}
\E S(i, n, j) \leq \frac{2 C^2 \beta \log(\bar{c} n P^{-d})}{[\Delta_{j}^i]^2} + \frac{\beta + 2}{\beta-2}.
\end{equation}
Define~$\underline{\Delta}_j:=\min_{i \in I_j^0}  \Delta_j^i$. We note that~$\underline{\Delta}_j > 0$ follows from inserting~$x = x_j$ in Equation~\eqref{eqn:step4bfirst}, and from using the definition of~$I_{j}^0$. Next, noting that~$\max_{i \in I_j^0} \Delta_j^i \leq 2C$ by Assumption~\ref{as:MAIN}, and combining Equations~\eqref{eqn:perarmregr} and~\eqref{eqn:yaclaim}, we obtain the bound
\begin{align}\label{eqn:4bbound}
\E[\tilde{R}_{j,I_j^0}(\bar{\pi})] \leq 
K \frac{2 C^2 \beta \log(\bar{c} n P^{-d})}{\underline{\Delta}_j} \left(1 + \frac{c_1 P^{-\gamma}}{\underline{\Delta}_j}\right)
+ (c_1+2C) K \frac{\beta + 2}{\beta-2}.
\end{align}

It remains to prove the claim in Equation~\eqref{eqn:yaclaim}. To this end we apply a conditioning argument as in the proof of Theorem~\ref{Thm:Covariate}. We shall now use some quantities (in particular the sets~$\Omega(v)$) that were defined in that proof: Note that
\begin{equation}\label{eqn:splitup}
\E S(i, n, j) = \sum_{v \in \{0, 1\}^n} \mathbb{P}(\Omega(v)) \mathbb{E}(S(i, n, j) | \Omega(v)).
\end{equation}
Arguing as in the proof of Theorem~\ref{Thm:Covariate}, it is now easy to see that~$\mathbb{E}(S(i,n,j)|\Omega(v))$ can be written as the expected number of times treatment~$i$ is selected in running the F-UCB policy~$\hat{\pi}$ (without covariates) in a problem with~$\bar{m} = \sum_{s = 1}^n v_s$ (fixed) i.i.d.~inputs with distribution~$\mathbb{Q}$ (the marginals of which have a cdf that lies in the closure of~$\mathscr{D}$ w.r.t.~$\|\cdot\|_{\infty}$ as a consequence of Lemma~\ref{lem:convex}). We can hence (cf.~Remark~2.4 in~\cite{kpv1}) apply the bound established in Equation~(26) of~\cite{kpv1}, to the just mentioned problem, to obtain
\begin{equation}
\mathbb{E}(S(i,n,j)|\Omega(v)) \leq \frac{2 C^2 \beta \log(\bar{m})}{[\Delta^{i}_j]^2} + \frac{\beta + 2}{\beta - 2}.
\end{equation}
We can now combine the obtained inequality with Equation~\eqref{eqn:splitup} to see that
\begin{equation}
\E S(i, n, j)\leq \sum_{v \in \{0, 1\}^n} \mathbb{P}(\Omega(v)) \frac{2 C^2 \beta \log(\bar{m})}{[\Delta^{i}_j]^2} + \frac{\beta + 2}{\beta - 2}.
\end{equation}
The claim in~\eqref{eqn:yaclaim} now follows from Jensen's inequality, and (cf.~the end of the proof of Theorem~\ref{Thm:Covariate})~$\sum_{v \in \{0, 1\}^n} \mathbb{P}(\Omega(v)) \bar{m} \leq \bar{c}n P^{-d}$.

\noindent
\textbf{Step~4c: A bound for~$\mathbb{E}(\tilde{R}_j(\bar{\pi}))$ with~$j \in \mathcal{J}$.} For all~$i \in I_j^0$ and all~$x \in B_j$ the triangle inequality and Lemma~\ref{ErrorHolder} with~$V_j=\sqrt{d} P^{-1}$ shows that~$c_1 P^{-\gamma}$ is smaller than 
\begin{align*}
&|\mathsf{T}(F^{\pi^{\star}(x_j)}(\cdot,x_j))-\mathsf{T}(F^{i}(\cdot,x_j))| \\
\leq &  |\mathsf{T}(F^{\pi^{\star}(x_j)}(\cdot,x_j))-\mathsf{T}(F^{\pi^{\star}(x)}(\cdot,x))| +|\mathsf{T}(F^{\pi^{\star}(x)}(\cdot,x))-\mathsf{T}(F^{i}(\cdot,x))|+ |\mathsf{T}(F^{i}(\cdot,x))-\mathsf{T}(F^{i}(\cdot,x_j))| \\
\leq & 2 C L d^{\gamma/2} P^{-\gamma}+|\mathsf{T}(F^{\pi^{\star}(x)}(\cdot,x))-\mathsf{T}(F^{i}(\cdot,x))|.
\end{align*}
Recalling that~$c_1 = 4 C L d^{\gamma/2}+1$, we obtain  
\begin{equation}\label{LowerBoundDelta}
\mathsf{T} \big(F^{\pi^{\star}(x)}(\cdot,x) \big)-\mathsf{T}(F^{i}(\cdot,x))>(1+2 C L d^{\gamma/2}) P^{-\gamma}.
\end{equation}
[In particular, since~$I_j^0 \neq \emptyset$ holds,~$0< \mathsf{T}(F^{\pi^{\star}(x)}(\cdot,x))-\mathsf{T}(F^{\pi^{\sharp}(x)}(\cdot,x))$ for all~$x\in B_j$ if~$j \in \mathcal{J}$, an observation we shall need later in Step~4d.] For every~$i \in I_j^0$ and every~$x \in B_j$,~\eqref{LowerBoundDelta} and Lemma~\ref{ErrorHolder} (recalling that~$\P_X(B_j) > \underline{c}P^{-d} > 0$) imply 
\begin{equation}\label{eqn:lowdelta}
\Delta_{j}^i=\mathsf{T} (F_j^{\ast})-\mathsf{T} (F_j^{i}) \geq \mathsf{T}(F^{\pi^{\star}(x)}(\cdot,x))-\mathsf{T}(F^{i}(\cdot,x))-2 C L d^{\gamma/2}P^{-\gamma} > P^{-\gamma};
\end{equation} 
in particular, for any~$j \in \mathcal{J}$, we have~$\underline{\Delta}_j = \min_{i \in I_j^0}  \Delta_j^i > P^{-\gamma}$. Recalling that~$\tilde{R}_j(\bar{\pi}) = \tilde{R}_{j, I_j^*}(\bar{\pi}) + \tilde{R}_{j, I_j^0}(\bar{\pi})$, we combine~\eqref{eqn:4abound} and~\eqref{eqn:4bbound} (with the just observed~$\underline{\Delta}_j> P^{-\gamma}$) to see  that for any~$j \in \mathcal{J}$ 
\begin{align}\label{eqn:boundforsome}
\mathbb{E}[\tilde{R}_j(\bar{\pi})]\leq
n c_1 P^{-\gamma} q_j +\frac{2 C^2 (c_1+1) K \beta \log(\bar{c} n P^{-d}) }{\underline{\Delta}_j}+(c_1 + 2C) K \frac{\beta+2}{\beta-2}.
\end{align}

\noindent
\textbf{Step~4d: A bound for~$\sum_{j \in \mathcal{J}} \mathbb{E}[\tilde{R}_j(\bar{\pi})]$.} Using Equation~\eqref{eqn:boundforsome} and~$|\mathcal{J}|\leq P^d$ we obtain
\begin{equation}\label{RegretSumWell}
\sum_{j \in \mathcal{J}} \mathbb{E}[\tilde{R}_j(\bar{\pi})]  \leq (c_1 + 2C) K \frac{\beta+2}{\beta-2} P^d+n c_1 P^{-\gamma} \sum_{j \in \mathcal{J}} q_j
+\sum_{j \in \mathcal{J}} \frac{2 C^2 (c_1+1) K \beta \log(\bar{c} n P^{-d}) }{\underline{\Delta}_j}.
\end{equation}
Since the~$B_j$ are disjoint, we obtain, recalling the definition of~$q_j$ after Equation~\eqref{eqn:4abound}, that 
\begin{align}\label{eqn:oneofmany}
\frac{n c_1}{P^{\gamma}} \sum_{j \in \mathcal{J}} q_j \leq 
\frac{n c_1}{P^{\gamma}} \mathbb{P} \big( 0< \mathsf{T}(F^{\pi^{\star}(X_1)}(\cdot,X_1))-\mathsf{T}(F^{\pi^{\sharp}(X_1)}(\cdot,X_1))<c_1 P^{-\gamma} \big) \leq C_0 c_1^{1+\alpha} n P^{-\gamma(1+\alpha)},
\end{align}
where we used~Assumption~\ref{ass:Margin} to obtain the last inequality.

To deal with the last sum  in the upper bound in~\eqref{RegretSumWell}, we need a better lower bound on the~$\underline{\Delta}_j$-s than the already available~$P^{-\gamma}$. For notational simplicity, let's suppose that the well-behaved bins are indexed as~$\mathcal{J}=\{1, 2, \ldots, j_1 \}$ such that 
$0<P^{-\gamma} \leq \underline{\Delta}_1 \leq \underline{\Delta}_2 \leq \ldots \leq \underline{\Delta}_{j_1}$.
Fix~$j \in \mathcal{J}$. Then, for any~$k = 1, \hdots, j$, we claim that:
\begin{align}\label{PointwiseIneq}
B_k \subseteq 
\left\{x: 0< \mathsf{T}(F^{\pi^{\star}(x)}(\cdot,x))-\mathsf{T}(F^{\pi^{\sharp}(x)}(\cdot,x))<\underline{\Delta}_j+2 C L d^{\gamma/2} P^{-\gamma} \right\}.
\end{align}
To see \eqref{PointwiseIneq}, note that, by definition, there exists an~$i \in \mathcal{I}_k^0$ such that 
$\underline{\Delta}_k=\mathsf{T}(F_k^{*})-\mathsf{T}(F_k^{i}).$ Given~$x \in B_k$, Lemmas~\ref{ErrorHolder} and~\ref{lem:convex} and Remark 2.4 in~\cite{kpv1} yield (the first inequality following from the observation after Equation~\eqref{LowerBoundDelta})
\begin{align*}
0< \mathsf{T}(F^{\pi^{\star}(x)}(\cdot,x))-\mathsf{T}(F^{\pi^{\sharp}(x)}(\cdot,x)) &\leq \mathsf{T}(F^{\pi^{\star}(x)}(\cdot,x))-\mathsf{T}(F^{i}(\cdot,x)) \\
& \leq \underline{\Delta}_k+2 C L d^{\gamma/2} P^{-\gamma} \\
&\leq \underline{\Delta}_j+2 C L d^{\gamma/2} P^{-\gamma},
\end{align*}
and thus~$x$ is an element of the set on the right-hand-side of~\eqref{PointwiseIneq}. Since all bins~$B_k$ are disjoint and~$\underline{\Delta}_j+2 C L d^{\gamma/2}P^{-\gamma}\leq c_1 \underline{\Delta}_j$ (obtained by recalling~$c_1 = 4 C L d^{\gamma/2}+1$, and using~$\underline{\Delta}_j > P^{-\gamma}$), the inclusion~\eqref{PointwiseIneq} yields that for any~$j \in \mathcal{J}$:
\begin{align}\label{DeltajIneq1}
\mathbb{P}_X \big(x: 0< \mathsf{T}(F^{\pi^{\star}(x)}(\cdot,x))-\mathsf{T}(F^{\pi^{\sharp}(x)}(\cdot,x))<c_1 \underline{\Delta}_j \big) 
&\geq \sum_{k=1}^{j} \mathbb{P}_X(B_k)  \geq \frac{\underline{c} j}{P^d}.
\end{align}

Let's denote~$j_2:=\max \{j \in \mathcal{J}: \underline{\Delta}_{j} \leq 1/c_1 \}$ (here interpreting the maximum of an empty set as~$0$).
Then, for each~$j \in \{1, \ldots, j_2 \}$ by Assumption~\ref{ass:Margin} :
\begin{align}\label{DeltajIneq2}
\mathbb{P}_X \big( 0< \mathsf{T}(F^{\pi^{\star}(X)}(\cdot,X))-\mathsf{T}(F^{\pi^{\sharp}(X)}(\cdot,X))<c_1 \underline{\Delta}_j \big)
\leq  C_0 (c_1 \underline{\Delta}_j)^{\alpha}.
\end{align}
Combining~\eqref{DeltajIneq1},~\eqref{DeltajIneq2}, and~$\underline{\Delta}_j > P^{-\gamma}$, for any~$j \in \{1, \ldots, j_2 \}$ we get 
$\underline{\Delta}_j \geq \max \big( c_* \big(j P^{-d } \big)^{1/\alpha}, P^{-\gamma} \big)$, with constant~$c_* := c_1^{-1} \underline{c}^{1/\alpha} C_0^{-1/\alpha}$. 
Combining this with the identity~$\underline{\Delta}_j>1/c_1$ for~$j>j_2$, we obtain that
\begin{align*}
\sum_{j \in \mathcal{J}} \frac{1}{\underline{\Delta}_j} & \leq 
\sum_{j=1}^{j_2} \min \left(c_*^{-1} \big(P^d/j \big)^{1/\alpha}, P^{\gamma} \right)+\sum_{j=j_2+1}^{j_1} c_1 
\leq \sum_{j=1}^{P^d} \min \left( c_*^{-1} \big(P^d/j \big)^{1/\alpha}, P^{\gamma} \right)+ c_1 P^d.
\end{align*}
For~$\tilde{P} := \lceil P^{d-\alpha \gamma} \rceil$ (in fact for \emph{any}~$\tilde{P} \in \{1, \hdots, P^d\}$, and thus in particular for our particular choice) it holds that 
\begin{align*}
\sum_{j=1}^{P^d} \min \left(c_*^{-1} \big(P^d/j \big)^{1/\alpha}, P^{\gamma} \right)\leq\sum_{j=1}^{\tilde{P}} P^{\gamma}+ 
c_*^{-1} P^{d/\alpha}\sum_{j=\tilde{P}+1}^{\infty} j^{-1/\alpha} \leq c_{**} P^{d+\gamma(1-\alpha)},
\end{align*}
for~$c_{**} := [2+c_*^{-1}(\alpha^{-1}-1)^{-1}]$, where we used~$\sum_{j=\tilde{P}+1}^{\infty} j^{-1/\alpha}\leq (\alpha^{-1} - 1)^{-1} \tilde{P}^{1-\alpha^{-1}}$. Hence,  Equations~\eqref{RegretSumWell} and~\eqref{eqn:oneofmany}, and the bounds in the previous two displays imply
\begin{align}\label{RegWellBeh}
\sum_{j \in \mathcal{J}} \mathbb{E}[\tilde{R}_j(\bar{\pi})] \leq 
c{'''} \left( n P^{-\gamma(1+\alpha)}+ K \overline{\log}(n P^{-d}) P^d+ K \overline{\log}(n P^{-d}) P^{d+\gamma(1-\alpha)} \right),
\end{align}
for a constant~$c{'''}$, say, that depends on~$d, L, \gamma, \underline{c}, \overline{c}, C, C_0, \alpha$ and~$\beta$, but \emph{not} on~$n$. 

\medskip
\noindent
\textbf{Step~5: Combining.}
From Equations~\eqref{eqn:Regretthreetypes}, \eqref{RegStrIllBeh},~\eqref{RegWeakIllBeh} and~\eqref{RegWellBeh} we obtain
\begin{equation}
\mathbb{E}[R_n(\bar{\pi})]  \leq \frac{c''''}{4} \left( n P^{-\gamma(1+\alpha)}+ \sqrt{K n \overline{\log}(n)} P^{d/2-\gamma \alpha}+ K \overline{\log}(n P^{-d}) P^d 
+ K \overline{\log}(n P^{-d}) P^{d+\gamma(1-\alpha)} \right) \end{equation}
for a constant~$c''''$ that depends on~$d, L, \gamma, \underline{c}, \bar{c}, C, C_0, \alpha$ and~$\beta$, but \emph{not} on~$n$. From~$P = \lceil n^{1/(2 \gamma +d)}  \rceil$ we get~$n \leq P^{2\gamma + d}$, and obtain
\begin{align}
\mathbb{E}[R_n(\bar{\pi})]   &\leq \frac{c''''}{4} K \overline{\log}(n) \left( n  P^{-\gamma(1+\alpha)}+ n^{1/2} P^{d/2-\gamma \alpha}+ 2 P^{d+\gamma(1-\alpha)} \right) \leq 
c'''' K \overline{\log}(n) P^{d+\gamma(1-\alpha)},
\end{align}
from which the conclusion follows.

\subsection{Proof of Theorem~\ref{thm:ISR}}

To prove the theorem we just combine Theorem~\ref{UBCovariateMargin} and the following lemma, which allows one to upper bound the number of suboptimal assignments made by any policy. 

\begin{lemma}\label{ISR}
Suppose Assumptions~\ref{as:MAIN} and~\ref{ass:Margin} hold. Let~$D_0 \geq \max(2, C_0^{-1})$, and define~$\tilde{C}(\alpha, D_0, C_0) = (1-1/D_0)/(C_0 D_0)^{1/\alpha}$.
Then, for any policy~$\pi$, any randomization measure, and for all~$(Y_t, X_t) \sim \P_{Y, X}$, such that~$\P_{Y, X}$ satisfies Equation~\eqref{eqn:marginDincls} and Assumption~\ref{ass:Holder}, it holds that
\begin{equation}\label{WTS1}
\mathbb{E}[R_n(\pi)] 
\geq
\tilde{C}(\alpha, D_0, C_0) n^{-1/\alpha} \big(\mathbb{E}[S_n(\pi)] \big)^{1+1/\alpha} \quad \text{ for every } n \in \N.
\end{equation}
\end{lemma}

\begin{remark}
In Lemma~\ref{ISR} we impose Assumptions~\ref{as:MAIN} and~\ref{ass:Holder} and Equation~\eqref{eqn:marginDincls} to guarantee that~$R_n(\pi)$ and~$S_n(\pi)$ are random variables, and that~$\pi^{\star}$ and~$\pi^\sharp$ are measurable, cf.~also the discussion in the footnote of Assumption~\ref{ass:Margin}.
\end{remark}
\begin{proof}
The proof-idea is quite standard and we follow~\cite{rigollet2010nonparametric}: Choose $D_0 \geq \max(2, C_0^{-1})$, implying that~$1/(C_0 D_0)^{1/\alpha}\leq 1$. Let~$n \in \N$, and let~$\pi$ be a policy as defined in Section~\ref{sec:setupX}. We write~$\pi_{n,t} = \pi_t$. Let~$\P_G$ be a randomization measure. We show that
\begin{equation}\label{WTS}
\mathbb{E}[R_n(\pi)] \geq \tilde{C} n^{-1/\alpha} \big(\mathbb{E}[S_n(\pi)] \big)^{1+1/\alpha}
\end{equation}
for~$\tilde{C}=\tilde{C}(\alpha, D_0, C_0)$. If~$\mathbb{E}[S_n(\pi)]=0$, \eqref{WTS} trivially holds. Thus, suppose that~$\mathbb{E}[S_n(\pi)]>0.$ Note that for any~$\delta>0$,
\begin{align*}
R_n(\pi) & \geq \delta \sum_{t=1}^n \mathds{1}_{ \{ \mathsf{T}(F^{\pi^\star(X_t)}(\cdot, X_t))-\mathsf{T}(F^{\pi^\sharp(X_t)}(\cdot, X_t)) > \delta  \} } \mathds{1}_{\cbr[1]{ \pi_t(X_t,Z_{t-1}, G_t)\not\in\argmax_{i \in \mathcal{I}}\cbr[0]{\mathsf{T}(F^i(\cdot,X_t))} }} \\
&= 
\delta S_n(\pi)-\delta \sum_{t=1}^n \mathds{1}_{ \{ \mathsf{T}(F^{\pi^\star(X_t)}(\cdot, X_t))-\mathsf{T}(F^{\pi^\sharp(X_t)}(\cdot, X_t)) \leq \delta  \} }\mathds{1}_{\cbr[1]{ \pi_t(X_t,Z_{t-1}, G_t)\not\in\argmax_{i \in \mathcal{I}}\cbr[0]{\mathsf{T}(F^i(\cdot,X_t))} }} \\
&=
\delta S_n(\pi)-\delta \sum_{t=1}^n \mathds{1}_{ \{ 0<\mathsf{T}(F^{\pi^\star(X_t)}(\cdot, X_t))-\mathsf{T}(F^{\pi^\sharp(X_t)}(\cdot, X_t)) \leq \delta  \} }\mathds{1}_{\cbr[1]{ \pi_t(X_t,Z_{t-1}, G_t)\not\in\argmax_{i \in \mathcal{I}}\cbr[0]{\mathsf{T}(F^i(\cdot,X_t))} }} \\
& \geq \delta S_n(\pi)-\delta \sum_{t=1}^n \mathds{1}_{ \{ 0<\mathsf{T}(F^{\pi^\star(X_t)}(\cdot, X_t))-\mathsf{T}(F^{\pi^\sharp(X_t)}(\cdot, X_t)) \leq \delta  \} },
\end{align*}
where the second equality used that if~$\pi_t(X_t,Z_{t-1}, G_t)\not\in\argmax_{i \in \mathcal{I}}\cbr[0]{\mathsf{T}(F^i(\cdot,X_t))}$, then~$0<\mathsf{T}(F^{\pi^\star(X_t)}(\cdot, X_t))-\mathsf{T}(F^{\pi^\sharp(X_t)}(\cdot, X_t))$. Choosing~$\delta:=(\mathbb{E}[S_n(\pi)]/(n C_0 D_0))^{1/\alpha} \leq 1/(C_0 D_0)^{1/\alpha}\leq1$ (the first inequality following from~$\mathbb{E}[S_n(\pi)] \leq n$), Assumption~\ref{ass:Margin} yields
\begin{align}\label{RnSnIneq}
\mathbb{E}[R_n(\pi)] \geq \delta (\mathbb{E}[S_n(\pi)]-C_0 n \delta^{\alpha})
=
\delta (1-1/D_0) \mathbb{E}[S_n(\pi)] 
= 
\tilde{C} n^{-1/\alpha} \big(\mathbb{E}[S_n(\pi)] \big)^{1+1/\alpha},
\end{align}
which proves~\eqref{WTS}. 
\end{proof}

\subsection{Proof of Theorem~\ref{LBCovMC}}

Let~$\pi$ be a policy, let~$\P_X$ be the uniform distribution on~$[0, 1]^d$, let~$\P_G$ be a randomization measure, and fix an~$n \in \N$. To simplify notation, we abbreviate~$\pi_{n, t} = \pi_t$. The proof of the inequalities in~\eqref{eqn:suplow} and~\eqref{eqn:suplow2} now proceeds in 5 steps:

\vspace{0.2cm}
\noindent
\textbf{Step~0: Preliminary observations and some notation.}
(a)~From the maintained assumptions and Assumption~\ref{as:MAIN} (imposed through Assumption~\ref{as:lbcov}) it follows that 
\begin{equation}\label{eqn:bilip}
c_- (\tau_2 - \tau_1) \leq \mathsf{T}(J_{\tau_2}) -\mathsf{T}(J_{\tau_1}) \leq C \| J_{\tau_2} - J_{\tau_1}\|_{\infty} \leq C (\tau_2 - \tau_1) \quad \text{ for every } \tau_1 \leq \tau_2 \text{ in } [0, 1]. 
\end{equation}
Let~$\varepsilon := 2/\sqrt{17} < 1/2$, set~$H_v := J_{1/2+v}$ for every~$v \in [-\varepsilon, \varepsilon]$, and define the map~$h: [-\varepsilon, \varepsilon] \to [h(-\varepsilon), h(\varepsilon)]$ via~$v\mapsto \mathsf{T}(H_v)$; note that~$h$ is strictly increasing because of~$c_->0$ and the observation in the previous display. (b)~The previous display also implies that~$h$ is Lipschitz continuous with constant~$C$, and that~$h(w) - h(v) \geq c_-(w-v)$ for every~$v\leq w$ in~$[-\varepsilon, \varepsilon]$; implying that~$h$ possesses a Lipschitz-continuous inverse function~$h^{-1}:[h(-\varepsilon), h(\varepsilon)] \to [-\varepsilon, \varepsilon]$, say, with constant~$c_-^{-1}$. (c)~Note that the map~$v \mapsto H_v$ (as a map from~$[-\varepsilon, \varepsilon]$ to~$D_{cdf}([a,b])$ equipped with the supremum metric) is Lipschitz continuous with constant~$1$. (d)~Finally, we verify that for~$\zeta := c_-^{-1} (0.5^2-\varepsilon^2)^{-1/2}$ we have (recalling the notational conventions introduced in the first paragraph of the Appendix)
\begin{equation}\label{eqn:aslbcov}
\mathsf{KL}^{1/2}(\mu_{H_{v}},\mu_{H_{w}})
\leq \zeta \left(\mathsf{T}(H_{w}) - \mathsf{T}(H_{v})\right) \quad \text{ for every } v \leq w \text{ in } [-\varepsilon, \varepsilon].
\end{equation}
By definition~$\mathsf{T}(H_{w}) - \mathsf{T}(H_{v}) = h(w) - h(v)$. Hence, the statement in~\eqref{eqn:aslbcov} follows from observation~(b) once we verify~$\mathsf{KL}^{1/2}(\mu_{H_{v}},\mu_{H_{w}}) \leq (w-v)/\sqrt{0.5^2-\varepsilon^2}$. But the latter is a simple consequence of Lemma~A.3 in~\cite{kpv1} (and is established similarly as the last claim in Lemma~A.4 in~\cite{kpv1}). 

\medskip

\vspace{0.2cm}
\noindent
\textbf{Step~1: Construction of a family of functions~$\mathcal{C}$.} For~$P \in \N$ (to be chosen in Step 4), let~$B_1^P, \ldots, B_{P^d}^P$ be the hypercubes defined in \eqref{SquareBins}, and sorted lexicographically; we shall drop the superscript~$P$ in the following. Let~$q_i,\ i=1,\hdots,P^d$, denote the center of~$B_i$. Let~$m:=\lceil P^{d-\gamma \alpha} \rceil$, and observe that~$1\leq m \leq P^d$. Next, let~$\Sigma_m:=\{-1, 1 \}^m$,~$|\Sigma_m|=2^m$, and define~$\mathcal{C}_m = \mathcal{C} := \{ f_{\sigma}: \sigma \in \Sigma_m \}$, where for~$\sigma \in \Sigma_m$ we construct~$f_{\sigma}:[0,1]^d \to \R$ via 
\begin{equation*}
f_{\sigma}(x):= h(0)+
c_-\varepsilon \sum_{j=1}^m \sigma_j \varphi_j(x);
\end{equation*}
for every~$j\in\cbr[0]{1,\hdots,P^d}$ we denote~$\varphi_j(x):=4^{-1} P^{-\gamma} \phi( 2 P(x-q_j))\mathds{1}_{B_j}(x)$, where~$\phi(x):=(1-||x||_{\infty})^{\gamma},$ and~$\|x\|_{\infty}:=\max_{1\leq i\leq d}|x_i|$ for~$x\in\mathbb{R}^d$. Note that every~$f_{\sigma}$ is continuous.

We now show that every~$f_{\sigma}$ is Hölder continuous. More precisely, we show that for every~$f_{\sigma}\in\mathcal{C}$ 
\begin{equation}\label{eqn:fholder}
|f_{\sigma}(x_1)-f_{\sigma}(x_2)| \leq  c_-\varepsilon 2^{-1} ||x_1-x_2||^{\gamma} \quad \text{ for every }  x_1, x_2 \in [0,1]^d,
\end{equation}
with~$\|\cdot\|$ denoting the Euclidean norm. We note that for any pair~$x_1, x_2 \in [0,1]^d$ one has~$|\phi(x_1)-\phi(x_2)| \leq ||x_1-x_2||_{\infty}^{\gamma} \leq ||x_1-x_2||^{\gamma}$; the second inequality is obvious, and the first inequality follows from~$|p^{\gamma}-q^{\gamma}| \leq |p-q|^{\gamma}$ for~$p,q\geq 0$ and~$0 < \gamma \leq 1$, together with the reverse triangle inequality. Now, to show~\eqref{eqn:fholder}, we consider two cases: First, if~$x_1, x_2 \in B_j$ for~$j\in\cbr[0]{1,\hdots,P^d}$, the definition of~$f_\sigma$ and~$|\phi(x_1)-\phi(x_2)|  \leq ||x_1-x_2||^{\gamma}$ lead to (note that if~$j > m$, the following inequality trivially holds)
\begin{align}\label{SameBin}
[c_-\varepsilon]^{-1} |f_{\sigma}(x_1)-f_{\sigma}(x_2)|  \leq |\varphi_j(x_1)-\varphi_j(x_2)| \leq  \frac{2^{\gamma}}{4} ||x_1-x_2||^{\gamma} \leq \frac{1}{2} ||x_1-x_2||^{\gamma}.
\end{align}
We remark that by continuity of~$f_{\sigma}$, equation~\eqref{SameBin} continues to hold if~$x_1$ and~$x_2$ are elements of the closure of~$B_j$, i.e., of~$\bar{B}_j$. Secondly, suppose that~$x_1 \in B_j, x_2 \in B_k$ for~$j \neq k$. Let~$S:=\{ \theta x_1+(1-\theta) x_2: \theta \in [0,1] \}$. Define~$y_1:=\text{argmin}_{z \in S \cap \bar{B}_j} ||z-x_2||$ and~$y_2:=\text{argmin}_{z \in S \cap \bar{B}_k} ||z-x_1||.$ Clearly,~$y_1$ and~$y_2$ are elements of the boundary of~$B_j$ and~$B_k$, respectively, implying~$\varphi_j(y_1)=\varphi_k(y_2)=0$. Denote~$\bar{\sigma}_i = \sigma_i$ for~$i = 1, \hdots, m$ and~$\bar{\sigma}_i = 0$ for~$i > m$. We obtain 
\begin{align*}
[c_-\varepsilon]^{-1}|f_{\sigma}(x_1)-f_{\sigma}(x_2)|=|\bar{\sigma}_j \varphi_j(x_1)-\bar{\sigma}_k \varphi_k(x_2)| 
&\leq | \varphi_j(x_1)- \varphi_j(y_1)|+
| \varphi_k(y_2)-\varphi_k(x_2)|\\
& \leq \frac{2^{\gamma}}{4}( ||x_1-y_1||^{\gamma}+||y_2-x_2||^{\gamma}	 )\\
&\leq 2^{-1} ||x_1-x_2||^{\gamma},											
\end{align*}
where for the second inequality we made use of the second inequality in~\eqref{SameBin} (cf.~also the remark immediately after~\eqref{SameBin}), and for the third inequality we combined 
$(a^{\gamma}+b^{\gamma}) \leq  2^{1-\gamma}(a+b)^{\gamma}$ for~$0 < \gamma \leq 1$ and~$a,b \geq 0$ with~$||x_1-y_1||+||y_2-x_2|| \leq ||x_1-y_1||+||y_1-y_2||+||y_2-x_2|| = ||x_1-x_2||$. Since the hypercubes~$B_1, \hdots, B_{P^d}$ define a partition of~$[0, 1]^d$ this establishes Equation~\eqref{eqn:fholder}.

\vspace{0.5cm}

\noindent
\textbf{Step~2: Construction of probability measures~$\mathbb{P}_f$ indexed by~$\mathcal{C}$.} Recall from Observation~(b) in Step 0 that~$h: [-\varepsilon, \varepsilon] \to [-h(\varepsilon), h(\varepsilon)]$ defined via~$v\mapsto \mathsf{T}(H_v)$ permits a Lipschitz-continuous inverse~$h^{-1}: [h(-\varepsilon), h(\varepsilon)] \to [-\varepsilon, \varepsilon]$, say, with corresponding Lipschitz constant~$c_-^{-1}$. By construction, the range of~$f\in\mathcal{C}$ is contained in~$[h(-\varepsilon), h(\varepsilon)]$, because~$h(\varepsilon) - h(0) \geq c_- \varepsilon$ and similarly~$h(0) - h(-\varepsilon) \geq c_- \varepsilon$. Hence, for every~$f \in \mathcal{C}$ the composition~$A_f := h^{-1} \circ f: [0, 1]^d \to [-\varepsilon, \varepsilon]$ is well-defined, and Equation~\eqref{eqn:fholder} shows that~$A_f$ is Hölder-continuous with constant~$\varepsilon /2$ and exponent~$\gamma$. Note that by definition
\begin{equation}\label{eqn:connectfT}
f(x) = h\left(h^{-1} \circ f (x))\right) = h(A_f(x)) =  \mathsf{T}\left(H_{A_f(x)}\right) \quad \text{ for every } x \in [0, 1]^d \text{ and every } f \in \mathcal{C}.
\end{equation}

We next show that~$\mu_{H_{A_f(\cdot)}} (\cdot): \mathcal{B}(\R) \times [0,1]^d \to [0,1]$, defined via~$B\times x \mapsto \mu_{H_{A_f(x)}}(B)$, is a stochastic kernel: (i) By definition,~$\mu_{H_{A_f(x)}}$ is a probability measure for every~$x \in [0, 1]^d$. (ii) Recall from Observation~(c) in Step~0 that~$\|H_{v} - H_{w}\|_{\infty} \leq |v - w|$ for every pair~$v, w \in [-\varepsilon, \varepsilon]$. From continuity of~$A_f$ it follows that~$x \mapsto H_{A_f(x)}(c) = \mu_{H_{A_f(x)}}((-\infty, c])$ is continuous (and hence measurable) for every~$c \in \R$. Since~$\{(-\infty, c]: c \in \R\}$ is a~``$\pi$-system'' that generates the Borel~$\sigma$-algebra on~$\R$, Lemma~1.40 of \cite{kallenberg} shows that~$\mu_{H_{A_f(\cdot)}} (\cdot): \mathcal{B}(\R) \times [0,1]^d \to [0,1]$ is a stochastic kernel. 

For every~$f\in\mathcal{C}$, we define the probability measure~
\begin{equation}\label{eqn:Pfdef}
\P_f := \mu_{H_0} \otimes [\mu_{H_{A_f(\cdot)}} \otimes \P_X];
\end{equation}
noting that the product in brackets is a semi-direct product. For later reference, we note that if~$(Y_t, X_t) \sim \mathbb{P}_f$, it holds for every~$x \in [0, 1]^d$ that~$F^1(\cdot, x) = H_0$ and~$F^2(\cdot, x) = H_{A_f(x)}$. In particular, Equation~\eqref{eqn:marginDincls} is satisfied as a consequence of Assumption~\ref{as:lbcov}. Now, for every~$t = 1, \hdots, n$, denote by~$\P_{\pi, f}^t$ the probability measure on the Borel sets of~$\R^{(d+2)t}$ induced by the (recursively defined) random vector~$Z_t = (Y_{\pi_{t}(X_{t}, Z_{t-1}, G_t), t}, X_{t}, G_t, \hdots, Y_{\pi_{1}(X_{1}, G_{1}), 1}, X_{1}, G_1)$ with i.i.d.~$(Y_t, X_t, G_t) \sim \P_f \otimes \P_G$. In the sequel, for~$t = 1, \hdots, n$, the symbol~$z_t$ will denote a ``generic'' element of~$\R^{(d+2)t}$ (i.e., a ``realization'' of the random vector~$Z_t$). 

We close this step with an important observation: Note that~$\bar{K}_{t,f}: \mathcal{B}(\R) \times [0,1]^d \times \R \times \R^{(t-1)(d+1)}$ defined via
\begin{equation}\label{eqn:defKbar}
B \times x \times g  \times z_{t-1} \mapsto  
\mu_{H_0}(B) \mathds{1}\{\pi_{t}(x, z_{t-1}, g) = 1\} + \mu_{H_{A_f(x)}}(B) \mathds{1}\{\pi_{t}(x,  z_{t-1}, g) = 2\} 
\end{equation}
is a regular conditional distribution of~$Y_{\pi_t(X_t, Z_{t-1}, G_t), t}$ given~$(X_t, G_t, Z_{t-1})$, and that for every~$t=1, \hdots, n$ we can therefore write (noting that~$Z_t = (Y_{\pi_t(X_t, Z_{t-1}, G_t), t}, X_t, G_t, Z_{t-1})$, interpreting~$Z_0$ as the empty vector)
\begin{equation}\label{eqn:kernelKULLprep}
\P_{\pi, f}^t = \bar{K}_{t,f} \otimes [\P_X \otimes \P_G \otimes \P_{\pi, f}^{t-1} ] ,
\end{equation}
with the convention that in case~$t = 1$ one has to drop the factor~$\P_{\pi, f}^{t-1}$ in the previous display and the ``$z_{t-1}$'' in Equation~\eqref{eqn:defKbar}. Hence, interpreting~$\mathsf{KL}(\P_{\pi, f_1}^{t-1}, \P_{\pi, f_2}^{t-1}) = 0$ in  case~$t = 1$, and with the just mentioned ``dropping''-convention, the Chain Rule of Lemma~\ref{lem:CHAIN} implies that for~$f_1, f_2 \in \mathcal{C}$ and any~$t =1,\hdots, n$ we have 
\begin{align*}
\mathsf{KL}	(
\P_{\pi, f_1}^t, \P_{\pi, f_2}^t
)
= &\mathsf{KL} \left( 
\bar{K}_{t,f_1} \otimes [\P_X \otimes \P_G \otimes \P_{\pi, f_1}^{t-1}  ],
\bar{K}_{t,f_2} \otimes [\P_X \otimes  \P_G \otimes \P_{\pi, f_2}^{t-1}  ]
\right)
 \\
=&\mathsf{KL}(\P_{\pi, f_1}^{t-1}, \P_{\pi, f_2}^{t-1}) + 
\mathsf{KL}	\left( 
\bar{K}_{t,f_1} \otimes [\P_X \otimes \P_G \otimes \P_{\pi, f_1}^{t-1}],
\bar{K}_{t,f_2} \otimes [\P_X \otimes \P_G \otimes \P_{\pi, f_1}^{t-1}]
\right),
\end{align*}
the right-hand-side being equal to the sum of~$\mathsf{KL}(\P_{\pi, f_1}^{t-1}, \P_{\pi, f_2}^{t-1})$ and 
\begin{equation*}
\int_{[0, 1]^d \times \R \times \R^{(t-1)(d+2)} }
\mathsf{KL}(
\bar{K}_{t,f_1}(\cdot, x, g, z_{t-1}),
\bar{K}_{t,f_2}(\cdot, x, g, z_{t-1})
)
d(\P_X \otimes \P_G \otimes  \P_{\pi, f_1}^{t-1} )(x, g, z_{t-1}).
\end{equation*}
Using~Equation~\eqref{eqn:defKbar} this sum further simplifies to
\begin{equation*}
\mathsf{KL}(\P_{\pi, f_1}^{t-1}, \P_{\pi, f_2}^{t-1}) + 
\int_{\{\pi_{t} = 2\}}
\mathsf{KL}(\mu_{H_{A_{f_1}(x)}}, \mu_{H_{A_{f_2}(x)}})
d(\P_X \otimes \P_G \otimes \P_{\pi, f_1}^{t-1})(x, g, z_{t-1}),
\end{equation*}
which, noting that~$\P_{\pi, f_1}^{t-1}$ is obtained by a coordinate projection from~$\P_{\pi, f_1}^n$, implies
\begin{equation*}
\mathsf{KL}	\left( 
\P_{\pi, f_1}^t, \P_{\pi, f_2}^t
\right)
\leq
\mathsf{KL}(\P_{\pi, f_1}^{t-1}, \P_{\pi, f_2}^{t-1}) + 
\int_{\{\pi_{t} = 2\}}
\mathsf{KL}(\mu_{H_{A_{f_1}(x)}}, \mu_{H_{A_{f_2}(x)}})
d(\P_X \otimes \P_G \otimes \P_{\pi, f_1}^{n} )(x, g, z_n).
\end{equation*}
By induction, it now immediately follows that for every~$t = 1, \hdots, n$
\begin{equation}\label{eqn:KLforlater}
\mathsf{KL}	\left( 
\P_{\pi, f_1}^t, \P_{\pi, f_2}^t
\right) \leq \int \sum_{i = 1}^t \mathds{1}\{\pi_{i} = 2\}
\mathsf{KL}(\mu_{H_{A_{f_1}(x)}}, \mu_{H_{A_{f_2}(x)}})
d(\P_X \otimes \P_G \otimes \P_{\pi, f_1}^{n})(x, g, z_n).
\end{equation}

\vspace{0.5cm}
\noindent
\textbf{Step~3: Verifying Assumptions~\ref{ass:Holder}  and~\ref{ass:Margin} for every~$\mathbb{P}_f$.} Fix~$f = f_{\sigma} \in \mathcal{C}$. To verify Assumption~\ref{ass:Holder} (with~$\gamma$ and~$L= \varepsilon/2$ as given in the theorem, cf.~Step~0 for the definition of~$\eps$) for~$\P_f$, which was defined in~\eqref{eqn:Pfdef}, note that
\begin{align*}
\|F^2(\cdot, x_1) -  F^2(\cdot, x_2)\|_{\infty} =
\|H_{A_f(x_1)}-H_{A_f(x_2)}\|_\infty
\leq  |A_f(x_1)-A_f(x_2)| \leq L \|x_1 - x_2\|^{\gamma},
\end{align*}
the first inequality following Observation~(c) in Step 0, and the second following from~$A_f$ being Hölder-continuous with constant~$L=\varepsilon/2$ and exponent~$\gamma$, as observed in Step~2 right before Equation~\eqref{eqn:connectfT}; note further that~$F^1(\cdot, x) = H_{0}$, and that the previous display hence trivially holds for~$F^2$ replaced by~$F^1$. Next, to verify Assumption~\ref{ass:Margin} (with~$\alpha$ and~$C_0 = 8d[c_- \varepsilon]^{-\alpha}$ as given in the theorem), it suffices to show (recall that~$K = 2$) that 
\begin{equation}\label{MC}
\P_X\left(x \in [0, 1]^d:0<|\mathsf{T}(H_{A_f(x)})-\mathsf{T}(H_{0})| \leq c_-\varepsilon \delta \right) \leq 8 d \delta^{\alpha} \mbox{ for all } \delta \geq 0.
\end{equation}
The statement in~\eqref{MC} is trivial for~$\delta = 0$. Let~$\delta > 0$. We use Equation~\eqref{eqn:connectfT} to write~$$[c_-\varepsilon]^{-1} | \mathsf{T}(H_{A_f(x)})-\mathsf{T}(H_{0}) |= \sum_{j=1}^m  \varphi_j(x),$$ where we used that~$B_j\cap B_k=\emptyset$ for~$j\neq k$. Noting that~$\sum_{j=1}^m  \varphi_j(x) = 0$ for~$x \notin \bigcup_{j = 1}^m B_j$, we obtain 
\begin{align*}
\P_X\left(x \in [0, 1]^d:0<|\mathsf{T}(H_{A_f(x)})-\mathsf{T}(H_{0})| \leq  c_-\varepsilon \delta \right)&=\sum_{j=1}^m \P_X\left(x \in B_j: 0< \varphi_j(x) \leq \delta \right),
\end{align*}
which we can write as
\begin{align*}
m \P_X\left(x \in B_1:  \phi(2P(x-q_1)) \leq 4P^{\gamma} \delta \right) 
&= m (2 P)^{-d}\int_{[-1,1]^d}\mathds{1}_{\cbr[0]{\phi \leq 4P^\gamma\delta}}dx \\&= m P^{-d} \int_{[0,1]^d} \mathds{1}_{ \{ \phi \leq 4 P^{\gamma} \delta \}} dx,
\end{align*}
where the first equality follows upon  substituting~$u=2P(x-q_1)$, and the second equality follows from~$\phi(x)$ being invariant to multiplying coordinates of~$x$ by~$-1$. To upper-bound the expression to the right in the previous display we consider two cases: If~$4 P^{\gamma} \delta>1,$ then 
\begin{align*}
m P^{-d} \int_{[0,1]^d} \mathds{1}_{ \{ \phi \leq 4 P^{\gamma} \delta \}} dx 
=
m P^{-d} 
\leq
2 P^{-\gamma \alpha}
\leq 
8 \delta^{\alpha},
\end{align*}
where we used~$m=\lceil P^{d-\gamma \alpha} \rceil \leq P^{d-\gamma \alpha}+1 \leq 2 P^{d-\gamma \alpha}$ and~${\alpha} \in (0,1)$. On the other hand, if~$4 P^{\gamma} \delta \leq 1$, we write~$\mathds{1}_{ \{ \phi \leq 4 P^{\gamma} \delta \}} = 1- \mathds{1}_{ \{ 4 P^{\gamma} \delta < \phi\}} = 1- \mathds{1}_{ \{ \|\cdot\|_{\infty}  < 1-(4\delta)^{1/\gamma} P\}}$ to obtain
\begin{equation*}
m P^{-d} \int_{[0,1]^d} \mathds{1}_{ \{ \phi \leq 4 P^{\gamma} \delta \}} dx 
=m P^{-d} (1- \int_{[0,1]^d} \mathds{1}_{ \{ \|\cdot\|_{\infty}  < 1-(4\delta)^{1/\gamma} P\}} dx )
=m P^{-d}[1-(1-(4\delta)^{1/\gamma} P)^d],
\end{equation*}
which, using~$(1-(1-s)^d) \leq ds$ for~$s \in [0, 1]$,~$m \leq 2P^{d-\gamma\alpha}$,~$P\leq (4\delta)^{-1/\gamma}$ and~$\alpha \in (0, 1)$, is bounded from above by$$m P^{1-d} d (4 \delta)^{1/\gamma}
\leq 2 d P^{1-\alpha \gamma} (4 \delta)^{1/\gamma} 
\leq 2 d (4 \delta)^{\alpha} \leq 8 d \delta^{\alpha}.$$

\vspace{0,5cm}
\noindent
\textbf{Step~4: Lower bounding the suprema in Equations~\eqref{eqn:suplow} and~\eqref{eqn:suplow2}.} We start with~Equation~\eqref{eqn:suplow2}. We already know that for every~$f \in \mathcal{C}$ the measure~$\P_f$ satisfies the inclusion in Equation~\eqref{eqn:marginDincls} and Assumptions~\ref{ass:Holder}  and~\ref{ass:Margin}. It therefore suffices to verify
\begin{align}
\sup_{f\in\mathcal{C}}\E_{(\P_f \otimes \P_G)^n }\sbr[1]{S_n(\pi)} \geq  n^{1-\frac{\alpha\gamma}{d+2\gamma}} \big / 32
\label{eq:ISR_tuple},
\end{align}
where~$\E_{ (\P_f \otimes \P_G)^n }$ denotes the expectation w.r.t.~the product measure~$\bigotimes_{t=1}^n(\P_f \otimes \P_G)$ (here, we interpret, with some abuse of notation,~$S_n(\pi)$ as a function on the range space of~$(X_t, Y_t, G_t)$ for~$t = 1, \hdots, n$; and we shall denote a generic realization of~$(X_t, Y_t, G_t)$ by $(x_t, y_t, g_t)$ to make this convention explicit, where we sometimes drop the subindex~$t$, if no confusion can arise).

We first observe that for~$\mathbb{P}_{f_{\sigma}}$, denoting~$\bar{f}_{\sigma} := [ c_- \varepsilon]^{-1}[f_{\sigma} - h(0)] = \sum_{j=1}^m \sigma_j \varphi_j$, we have
\begin{align*}
S_n(\pi)&=\sum_{t=1}^n \mathds{1}\{ \mathsf{T}(F^{1}(\cdot, x_t)) \neq \mathsf{T}(F^{2}(\cdot, x_t)),~\pi^{\star}(x_t) \neq \pi_{t}(x_t,z_{t-1}, g_t) \} \\
&=
\sum_{t=1}^n \mathds{1}\{ \bar{f}_{\sigma}(x_t) \neq 0,~2\pi_{t}(x_t,z_{t-1}, g_t) -3 \neq \text{sign}(\bar{f}_{\sigma}(x_t)) \},
\end{align*}
where for the second equality we used that~$\pi^{\star}(x)=3/2 + \text{sign}(\bar{f}_{\sigma}(x))/2$ (with the convention that the sign of~$0$ is~$-1$), and
where we recalled from Equation~\eqref{eqn:connectfT} that~$\mathsf{T}(F^{1}(\cdot, x)) \neq \mathsf{T}(F^{2}(\cdot, x))$ is equivalent to~$\bar{f}_{\sigma}(x) \neq 0$. Noting that the random vectors~$X_t$,~$Z_{t-1}$, and~$G_t$ are independent, it follows that their joint distribution equals~$\P_X \otimes \P_{\pi, f_{\sigma}}^{t-1} \otimes \P_G$. Using Tonelli's theorem, writing~$\mathbb{E}_G$ for the expectation w.r.t.~$\P_G$, abbreviating~$2\pi_{t}(x,z_{t-1},g) -3 := \check{\pi}_t(x,z_{t-1},g)$, and noting that the~$t$-th summand in the previous display depends on~$z_t$ only via~$z_{t-1}$, we obtain
\begin{align}
\sup_{f\in\mathcal{C}}\E_{ (\P_f \otimes \P_G)^n }[S_n(\pi)] 
&=
\sup_{\sigma \in \Sigma_m} \sum_{t=1}^n \mathbb{E}_{\pi, f_{\sigma}}^{t-1} \mathbb{E}_G \sbr[1]{\mathbb{P}_X\del[1]{
x : \bar{f}_{\sigma}(x) \neq 0,~\check{\pi}_t(x,z_{t-1}, g_t) \neq \text{sign}(\bar{f}_{\sigma}(x))
}} \notag  \\
& \geq \sup_{\sigma \in \Sigma_m} \sum_{j=1}^m \sum_{t=1}^n \mathbb{E}_{\pi, f_{\sigma}}^{t-1} \mathbb{E}_G [
\mathbb{P}_X(
 x \in B_j: \check{\pi}_t(x,z_{t-1}, g_t) \neq \sigma_j 
)
]\notag \\
&\geq 
\frac{1}{2^m}\sum_{j=1}^m \sum_{t=1}^n \sum_{\sigma \in \Sigma_m} \mathbb{E}_{\pi, f_{\sigma}}^{t-1} \mathbb{E}_G [\mathbb{P}_X(
 x \in B_j: \check{\pi}_t(x,z_{t-1}, g_t) \neq \sigma_j 
)]\label{eq:LBISR},
\end{align}
where we used that~$m\leq P^d$ and~$\mathbb{P}_X(x \in B_j: \bar{f}_{\sigma}(x) = 0) = 0$ (and where we use a corresponding ``dropping''-convention for the index~$t = 1$ as introduced after Equation~\eqref{eqn:kernelKULLprep}). For every~$j\in\cbr[0]{1,\hdots,m}$ and~$t\in\cbr[0]{1,\hdots,n}$,
\begin{align*}
Q_t^j
&:=
\sum_{\sigma \in \Sigma_m} \mathbb{E}_{\pi, f_{\sigma}}^{t-1} \mathbb{E}_G [\mathbb{P}_X(
 x \in B_j: \check{\pi}_t(x,z_{t-1}, g) \neq \sigma_j 
)] \\
&=
\sum_{\sigma_{-j}\in \Sigma_{m-1}}\sum_{i\in\cbr[0]{-1,1}}\mathbb{E}_{\pi, f_{\sigma_{-j}^i}}^{t-1} \mathbb{E}_G [\mathbb{P}_X(
 x \in B_j: \check{\pi}_t(x,z_{t-1},g) \neq i 
)],
\end{align*}
where~$\sigma_{-j}:=(\sigma_1,\hdots,\sigma_{j-1},\sigma_{j+1},\hdots,\sigma_m)$ and~$\sigma_{-j}^i:=(\sigma_1,\hdots,\sigma_{j-1},i,\sigma_{j+1},\hdots,\sigma_m)$ for~$i\in\cbr[0]{-1,1}$. Define for every~$j \in \{1, \hdots, m\}$ the probability measure~$\P_X^j$ via~$\P_X^j(A):=\P_X(A \cap B_j)/\P_X(B_j)$ for~$A\in \mathcal{B}(\R^d)$, and let~$\E_X^j$ be the corresponding expectation operator. Recalling~$\P_X(B_j) = P^{-d}$, we obtain for any~$z_{t-1} \in \mathbb{R}^{(t-1)(d+1)}$ and any~$g \in \R$ that
\begin{equation*}
\P_X(
\{ x \in B_j: \check{\pi}_t(x,z_{t-1}, g) \neq i  \}) =\P_X^j(\{x: \check{\pi}_t(x,z_{t-1}, g)\neq i\})/P^d,
\end{equation*}
from which we see that the sum over~$i$ in the penultimate display coincides, for every~$\sigma_{-j}\in \Sigma_{m-1}$, with
\begin{equation}\label{eqn:ingrLBX1}
\frac{1}{P^d}\left(
\mathbb{E}_{\pi, f_{\sigma_{-j}^{-1}}}^{t-1}\mathbb{E}_G\E_X^j\mathds{1}_{\cbr[0]{\check{\pi}_t(x,z_{t-1}, g)=1}}+1-\mathbb{E}_{\pi, f_{\sigma_{-j}^{1}}}^{t-1}\mathbb{E}_G\E_X^j\mathds{1}_{\cbr[0]{\check{\pi}_t(x,z_{t-1},g)=1}}
\right) =: \frac{1}{P^d} e(\sigma, j, t).
\end{equation} 
Clearly,~$e(\sigma, j, t)$ is the sum of the Type~1 and Type~2 error of the test~$(x, z_{t-1}, g) \mapsto \mathds{1}_{\cbr[0]{\check{\pi}_t(x,z_{t-1}, g)=1}}$ for
$$H_0:\P_X^j\otimes\P^{t-1}_{\pi,f_{\sigma_{-j}^{-1}}}\otimes \P_G \quad \text{ against } \quad H_1: \P_X^j \otimes\P^{t-1}_{\pi,f_{\sigma_{-j}^{1}}}\otimes \P_G.$$
Using Theorem 2.2(iii) of \cite{tsybakov2009introduction}, we obtain
\begin{equation}\label{eqn:ingrLBX2}
\begin{aligned}
e(\sigma, j, t)
&\geq
\frac{1}{4}\exp\sbr[2]{-\mathsf{KL}\del[1]{
\P_X^j\otimes\P^{t-1}_{\pi,f_{\sigma_{-j}^{-1}}}\otimes \P_G,
\P_X^j \otimes\P^{t-1}_{\pi,f_{\sigma_{-j}^{1}}}\otimes \P_G
}} \\
&=
\frac{1}{4}\exp\sbr[2]{-\mathsf{KL}\del[1]{\P^{t-1}_{\pi,f_{\sigma_{-j}^{-1}}},\P^{t-1}_{\pi,f_{\sigma_{-j}^{1}}}}},
\end{aligned}
\end{equation}
the equality following, e.g., from the Chain Rule in Lemma~\ref{lem:CHAIN}. 

To upper bound~$\mathsf{KL}\del[1]{\P^{t-1}_{\pi,f_{\sigma_{-j}^{-1}}},\P^{t-1}_{\pi,f_{\sigma_{-j}^{1}}}}$, we will now apply~\eqref{eqn:KLforlater} with~$f_1 = f_{\sigma_{-j}^{-1}}$ and~$f_2 = f_{\sigma_{-j}^{1}}$. Note first that~$f_1(x) = f_2(x)$ for~$x \notin B_j$, and that~$(f_1(x), f_2(x)) = (h(0)- c_-\varepsilon \varphi_j(x),h(0) + c_-\varepsilon \varphi_j(x))$ for~$x \in B_j$, from which it follows from Equations~\eqref{eqn:aslbcov} (note that~$A_{f_1(x)} \leq A_{f_2(x)}$ follows from strict monotonicity of~$h^{-1}$, cf.~Step~0) and \eqref{eqn:connectfT} that
\begin{equation}
\mathsf{KL}(\mu_{H_{A_{f_1}(x)}}, \mu_{H_{A_{f_2}(x)}})
\leq
\begin{cases}
[2  \zeta c_- \varepsilon  \varphi_j(x)]^2 & \text{ if } x \in B_j, \\
0 & \text{ if } x \notin B_j.
\end{cases}
\end{equation}
Since~$[2  \zeta c_- \varepsilon  \varphi_j(x)]^2 \leq [ \zeta c_- \varepsilon 2^{-1} P^{-\gamma}]^2 =: \bar{r}P^{-2\gamma}$ holds for~$x \in B_j$, Equation~\eqref{eqn:KLforlater} delivers 
\begin{equation*}
\mathsf{KL}	( 
\P_{\pi, f_{\sigma_{-j}^{-1}}}^{t-1}, \P_{\pi, f_{\sigma_{-j}^{1}}}^{t-1}
) \leq \bar{r}P^{-2\gamma}
\int \sum_{i = 1}^{t-1} \mathds{1}\{G(i,j)\}
d(\P_X \otimes \P_G \otimes \P_{\pi, f_{\sigma_{-j}^{-1}}}^{n}  ) \leq \bar{r}P^{-2\gamma} N_{j, \sigma_{-j}},
\end{equation*}
with~$G(i,j) := \{(x, z_n, g): x \in B_j, \pi_{i}(x, z_{i-1}, g) = 2\}$,~$N_{j, \sigma_{-j}}:=\int \sum_{i = 1}^n \mathds{1}\{G(i,j)\} d(\P_X \otimes \P_G \otimes \P_{\pi, f_{\sigma_{-j}^{-1}}}^n)$. The dependence of~$N_{j,\sigma_{-j}}$ on~$\pi$ has been suppressed. In combination with Equations~\eqref{eqn:ingrLBX1} and~\eqref{eqn:ingrLBX2} we hence obtain
\begin{align*}
\sum_{t = 1}^n Q_t^j = \sum_{t = 1}^n \sum_{\sigma_{-j}\in \Sigma_{m-1}} \frac{1}{P^d} e(\sigma, j, t) &\geq
\sum_{t = 1}^n \sum_{\sigma_{-j}\in \Sigma_{m-1}} 
\frac{1}{4 P^d}\exp\sbr[2]{-\bar{r}P^{-2\gamma} N_{j, \sigma_{-j}}} \\
&=
\frac{n}{4P^d} \sum_{\sigma_{-j}\in \Sigma_{m-1}} \exp\sbr[2]{-\bar{r}P^{-2\gamma} N_{j, \sigma_{-j}}} \\
&\geq
2^{m-1} \frac{n}{4P^d}  \exp\sbr[2]{-\bar{r}P^{-2\gamma} \rho_j },
\end{align*}
the last inequality following from Jensen's inequality and~$\rho_j := 2^{1-m}
\sum_{\sigma_{-j}\in \Sigma_{m-1}}N_{j, \sigma_{-j}}$. Furthermore, from the definition of~$Q_t^j$, one directly obtains via Tonelli's theorem that
\begin{align*}
\sum_{t = 1}^n Q_t^j &= 
\sum_{t = 1}^n \sum_{\sigma_{-j}\in \Sigma_{m-1}}\sum_{i\in\cbr[0]{-1,1}}\mathbb{E}_{\pi, f_{\sigma_{-j}^i}}^{t-1} \mathbb{E}_G [\mathbb{P}_X(
 x \in B_j: \check{\pi}_t(x,z_{t-1},g_t) \neq i 
)] \\
&\geq
\sum_{t = 1}^n \sum_{\sigma_{-j}\in \Sigma_{m-1}} \mathbb{E}_{\pi, f_{\sigma_{-j}^{-1}}}^{n} \mathbb{E}_G [\mathbb{P}_X(
 x \in B_j: \pi_t(x,z_{t-1},g_t) = 2 
)]
\\
&=
\sum_{\sigma_{-j}\in \Sigma_{m-1}} \mathbb{E}_{\pi, f_{\sigma_{-j}^{-1}}}^{n} \mathbb{E}_G \sum_{t = 1}^n [\mathbb{P}_X(
 x \in B_j: \pi_t(x,z_{t-1},g_t) = 2 
)]
\\
&=
\sum_{\sigma_{-j}\in \Sigma_{m-1}}N_{j, \sigma_{-j}}
=  2^{m-1}  \rho_j.
\end{align*}
Combining the lower bounds in the previous two displays with~\eqref{eq:LBISR} yields
\begin{equation*}
\sup_{f\in\mathcal{C}}\E_{ (\P_f \otimes \P_G)^n }[S_n(\pi)]
\geq
\frac{1}{2^m}\sum_{j=1}^m \sum_{t=1}^n
Q_t^j
\geq \frac{1}{2}\sum_{j=1}^m\max\left(\frac{n}{4P^d}\exp\left[- \bar{r}P^{-2\gamma}  \rho_j
\right], \rho_j \right),
\end{equation*}
which can further be lower-bounded by
\begin{align*}
\frac{1}{4}\sum_{j=1}^m \left(
\frac{n}{4P^d}\exp\left[- \bar{r}P^{-2\gamma} \rho_j
\right] + \rho_j
\right)
&\geq
\frac{m}{4}\inf_{\rho \geq 0}
\left(
\frac{n}{4P^d}\exp\left[- \bar{r}P^{-2\gamma}  \rho
\right] + \rho
\right) \\
& \geq 
\frac{m}{4\bar{r}P^{-2\gamma}}\inf_{\rho \geq 0}
\left(
\frac{n\bar{r}}{4P^{d+2\gamma} }\exp\left[- \rho
\right] + \rho\right).
\end{align*}
This lower bound holds for any~$P \in \N$ and corresponding~$m=\lceil P^{d-\gamma \alpha}\rceil$. We now set~$P := \lceil( n \bar{r}/4)^{1/(d+2\gamma)} \rceil$, and can thus use~$w \exp(-\rho) + \rho \geq w$ for every~$\rho \geq 0$ and every~$0 < w \leq 1$ to lower bound the quantity in the last line of the previous display by
\begin{equation*}
\frac{mn}{16P^{d}}  \geq \frac{ P^{d-\gamma \alpha}n}{16P^{d}} = \frac{n}{16}   P^{-\gamma \alpha} \geq \frac{n}{16}  [(n\bar{r}/4)^{1/(d+2\gamma)}+1]^{-\gamma \alpha} \geq
\frac{n^{1-\frac{\alpha\gamma}{d+2\gamma}}}{16} [(\bar{r}/4)^{1/(d+2\gamma)}+1]^{-\gamma \alpha}.
\end{equation*}
By definition,~$\bar{r} = [ \zeta c_- \varepsilon 2^{-1}]^2 = [(0.5^2-\varepsilon^2)^{-1/2} \varepsilon 2^{-1}]^2$. Recalling~$\varepsilon = 2/\sqrt{17}$ implies~$\bar{r} = 4$. Thus, the lower bound in the previous display simplifies to
\begin{equation}
\frac{n^{1-\frac{\alpha\gamma}{d+2\gamma}}}{16} 2^{-\gamma \alpha} \geq n^{1-\frac{\alpha\gamma}{d+2\gamma}} \big / 32.
\end{equation}
This establishes Equation~\eqref{eqn:suplow2}. Finally, Lemma~\ref{ISR} (cf.~Step 3, which verifies the assumptions needed) with $D_0 = 2+C_0^{-1}$ shows that the lower bound established in Lemma~\ref{ISR} holds for the corresponding constant~$(1-(2+C_0^{-1})^{-1})/(2C_0 +1)^{1/\alpha} \geq 2^{-1}(2C_0 +1)^{-1/\alpha} \geq 2^{-(1+1/\alpha)}(C_0 +1)^{-1/\alpha}$. This version of Lemma~\ref{ISR} and the already established Equation~\eqref{eqn:suplow2} proves Equation~\eqref{eqn:suplow}.

\onehalfspacing

\bibliographystyle{ecta}	
\bibliography{mergereferences}		

\begin{thebibliography}{24}
\newcommand{\enquote}[1]{``#1''}
\expandafter\ifx\csname natexlab\endcsname\relax\def\natexlab#1{#1}\fi

\bibitem[\protect\citeauthoryear{Audibert and Tsybakov}{Audibert and
  Tsybakov}{2007}]{audibert2007fast}
\textsc{Audibert, J.-Y. and A.~B. Tsybakov} (2007): \enquote{Fast learning
  rates for plug-in classifiers,} \emph{Annals of Statistics}, 35, 608--633.

\bibitem[\protect\citeauthoryear{Besson and Kaufmann}{Besson and
  Kaufmann}{2018}]{besson2018doubling}
\textsc{Besson, L. and E.~Kaufmann} (2018): \enquote{What doubling tricks can
  and can't do for multi-armed bandits,} \emph{arXiv preprint
  arXiv:1803.06971}.

\bibitem[\protect\citeauthoryear{Cassel, Mannor, and Zeevi}{Cassel
  et~al.}{2018}]{cassel2018general}
\textsc{Cassel, A., S.~Mannor, and A.~Zeevi} (2018): \enquote{A general
  approach to multi-armed bandits under risk criteria,} \emph{arXiv preprint
  arXiv:1806.01380}.

\bibitem[\protect\citeauthoryear{Folland}{Folland}{1999}]{folland}
\textsc{Folland, G.~B.} (1999): \emph{Real Analysis: Modern Techniques and
  their Applications}, New York: Wiley.

\bibitem[\protect\citeauthoryear{Kallenberg}{Kallenberg}{2001}]{kallenberg}
\textsc{Kallenberg, O.} (2001): \emph{Foundations of Modern Probability}, New
  York: Springer Science \& Business Media, 2 ed.

\bibitem[\protect\citeauthoryear{Kitagawa and Tetenov}{Kitagawa and
  Tetenov}{2018}]{kitagawa2018should}
\textsc{Kitagawa, T. and A.~Tetenov} (2018): \enquote{Who should be treated?
  {E}mpirical welfare maximization methods for treatment choice,}
  \emph{Econometrica}, 86, 591--616.

\bibitem[\protect\citeauthoryear{Kock, Preinerstorfer, and Veliyev}{Kock
  et~al.}{2020}]{kpv1}
\textsc{Kock, A.~B., D.~Preinerstorfer, and B.~Veliyev} (2020):
  \enquote{Functional Sequential Treatment Allocation,} \emph{arXiv preprint
  v(6) arXiv:1812.09408}.

\bibitem[\protect\citeauthoryear{Kock and Thyrsgaard}{Kock and
  Thyrsgaard}{2017}]{kock2017optimal}
\textsc{Kock, A.~B. and M.~Thyrsgaard} (2017): \enquote{Optimal sequential
  treatment allocation,} \emph{arXiv preprint arXiv:1705.09952}.

\bibitem[\protect\citeauthoryear{Liese and Miescke}{Liese and
  Miescke}{2008}]{liese}
\textsc{Liese, F. and K.~J. Miescke} (2008): \emph{Statistical Decision
  Theory}, New York: Springer.

\bibitem[\protect\citeauthoryear{Maillard}{Maillard}{2013}]{maillard}
\textsc{Maillard, O.-A.} (2013): \enquote{Robust Risk-Averse Stochastic
  Multi-armed Bandits,} in \emph{Algorithmic Learning Theory}, ed. by S.~Jain,
  R.~Munos, F.~Stephan, and T.~Zeugmann, Berlin, Heidelberg: Springer Berlin
  Heidelberg, 218--233.

\bibitem[\protect\citeauthoryear{Mammen and Tsybakov}{Mammen and
  Tsybakov}{1999}]{mammen1999smooth}
\textsc{Mammen, E. and A.~B. Tsybakov} (1999): \enquote{Smooth discrimination
  analysis,} \emph{Annals of Statistics}, 27, 1808--1829.

\bibitem[\protect\citeauthoryear{Perchet and Rigollet}{Perchet and
  Rigollet}{2013}]{perchet2013multi}
\textsc{Perchet, V. and P.~Rigollet} (2013): \enquote{The multi-armed bandit
  problem with covariates,} \emph{Annals of Statistics}, 693--721.

\bibitem[\protect\citeauthoryear{Rigollet and Zeevi}{Rigollet and
  Zeevi}{2010}]{rigollet2010nonparametric}
\textsc{Rigollet, P. and A.~Zeevi} (2010): \enquote{Nonparametric bandits with
  covariates,} \emph{Proceedings of COLT}.

\bibitem[\protect\citeauthoryear{Sani, Lazaric, and Munos}{Sani
  et~al.}{2012}]{NIPS2012_4753}
\textsc{Sani, A., A.~Lazaric, and R.~Munos} (2012): \enquote{Risk-Aversion in
  Multi-armed Bandits,} in \emph{Advances in Neural Information Processing
  Systems 25}, ed. by F.~Pereira, C.~J.~C. Burges, L.~Bottou, and K.~Q.
  Weinberger, Curran Associates, Inc., 3275--3283.

\bibitem[\protect\citeauthoryear{Serfling}{Serfling}{1984}]{serflinggen}
\textsc{Serfling, R.~J.} (1984): \enquote{Generalized L-, M-, and
  R-Statistics,} \emph{Annals of Statistics}, 12, 76--86.

\bibitem[\protect\citeauthoryear{Shalev-Shwartz}{Shalev-Shwartz}{2012}]{shalev2012online}
\textsc{Shalev-Shwartz, S.} (2012): \enquote{Online learning and online convex
  optimization,} \emph{Foundations and Trends{\textregistered} in Machine
  Learning}, 4, 107--194.

\bibitem[\protect\citeauthoryear{Tran-Thanh and Yu}{Tran-Thanh and
  Yu}{2014}]{tran2014functional}
\textsc{Tran-Thanh, L. and J.~Y. Yu} (2014): \enquote{Functional bandits,}
  \emph{arXiv preprint arXiv:1405.2432}.

\bibitem[\protect\citeauthoryear{Tsybakov}{Tsybakov}{2004}]{tsybakov2004optimal}
\textsc{Tsybakov, A.~B.} (2004): \enquote{Optimal aggregation of classifiers in
  statistical learning,} \emph{Annals of Statistics}, 32, 135--166.

\bibitem[\protect\citeauthoryear{Tsybakov}{Tsybakov}{2009}]{tsybakov2009introduction}
---\hspace{-.1pt}---\hspace{-.1pt}--- (2009): \emph{Introduction to
  Nonparametric Estimation}, New York: Springer.

\bibitem[\protect\citeauthoryear{Vakili, Boukouvalas, and Zhao}{Vakili
  et~al.}{2018}]{vakili2018decision}
\textsc{Vakili, S., A.~Boukouvalas, and Q.~Zhao} (2018): \enquote{Decision
  Variance in Online Learning,} \emph{arXiv preprint arXiv:1807.09089}.

\bibitem[\protect\citeauthoryear{{Vakili} and {Zhao}}{{Vakili} and
  {Zhao}}{2016}]{7515237}
\textsc{{Vakili}, S. and Q.~{Zhao}} (2016): \enquote{Risk-Averse Multi-Armed
  Bandit Problems Under Mean-Variance Measure,} \emph{IEEE Journal of Selected
  Topics in Signal Processing}, 10, 1093--1111.

\bibitem[\protect\citeauthoryear{Woodroofe}{Woodroofe}{1979}]{woodroofe1979one}
\textsc{Woodroofe, M.} (1979): \enquote{A one-armed bandit problem with a
  concomitant variable,} \emph{Journal of the American Statistical
  Association}, 74, 799--806.

\bibitem[\protect\citeauthoryear{Yang, Zhu et~al.}{Yang
  et~al.}{2002}]{yang2002randomized}
\textsc{Yang, Y., D.~Zhu, et~al.} (2002): \enquote{Randomized allocation with
  nonparametric estimation for a multi-armed bandit problem with covariates,}
  \emph{The Annals of Statistics}, 30, 100--121.

\bibitem[\protect\citeauthoryear{Zimin, Ibsen-Jensen, and Chatterjee}{Zimin
  et~al.}{2014}]{zimin2014generalized}
\textsc{Zimin, A., R.~Ibsen-Jensen, and K.~Chatterjee} (2014):
  \enquote{Generalized risk-aversion in stochastic multi-armed bandits,}
  \emph{arXiv preprint arXiv:1405.0833}.

\end{thebibliography}

\end{document}